\documentclass[11pt]{article} 

\pdfoutput=1

\usepackage{etoolbox}
\newtoggle{arxiv}
\toggletrue{arxiv}

\usepackage[utf8]{inputenc}
\usepackage{mathrsfs}
\usepackage{microtype}
\usepackage{subfigure}
\usepackage{booktabs} 
\usepackage{longtable,makecell}
\usepackage{amsmath, amsfonts, amsthm, amssymb, graphicx}
\usepackage{mathtools,verbatim}
\usepackage{enumitem}
\usepackage{bm}
\usepackage{thm-restate}
\usepackage{natbib}
\usepackage{xspace,upgreek}
\usepackage[dvipsnames]{xcolor}

\usepackage{color-edits}

\addauthor{df}{ForestGreen}
\addauthor{ms}{orange}
\addauthor{kevin}{ForestGreen}

\usepackage[noend]{algpseudocode}
\usepackage{algorithm}

\iftoggle{arxiv}
{
	\usepackage[scaled=.90]{helvet}
	
	\usepackage[vmargin=1.00in,hmargin=1.00in,centering,letterpaper]{geometry}
	\usepackage{fancyhdr, lastpage}

	\setlength{\headsep}{.10in}
	\setlength{\headheight}{15pt}
	\cfoot{}
	\lfoot{}
	\rfoot{\sc Page\ \thepage\ of\ \protect\pageref{LastPage}}
	\renewcommand\headrulewidth{0.4pt}
	\renewcommand\footrulewidth{0.4pt}
}

\usepackage{hyperref}
\usepackage{cleveref}
\hypersetup{colorlinks,citecolor=blue}
\newtoggle{upperbound}
\togglefalse{upperbound}

\newcommand{\algcomment}[1]{\textcolor{blue!70!black}{\footnotesize{\texttt{\textbf{//
          #1}}}}}

\newcommand{\poly}{\mathrm{poly}}
\newcommand{\calX}{\mathcal{X}}

\numberwithin{equation}{section}

\newcommand{\fro}{\mathrm{F}}
\newcommand{\op}{\mathrm{op}}

\DeclarePairedDelimiter\ceil{\lceil}{\rceil}

{
 \theoremstyle{plain}
      
}
\creflabelformat{asm}{#2#1#3}

\theoremstyle{plain}
\newtheorem{thm}{Theorem}
\newtheorem{lem}{Lemma}[section]
\newtheorem{cor}{Corollary}

\theoremstyle{definition}
\newtheorem{defn}{Definition}[section]

\newtheorem{exmp}{Example}[section]

\newtheorem{rem}{Remark}[section]

\renewcommand{\Pr}{\mathbb{P}}
\newcommand{\Exp}{\mathbb{E}}

\newcommand{\KL}{\mathrm{KL}}

\newcommand{\rmd}{\mathrm{d}}

\newcommand{\dist}{\mathrm{dist}}

\newcommand{\R}{\mathbb{R}}
\newcommand{\I}{\mathbb{I}}

\DeclareMathOperator*{\argmax}{arg\,max}
\DeclareMathOperator*{\argmin}{arg\,min}

\def\ddefloop#1{\ifx\ddefloop#1\else\ddef{#1}\expandafter\ddefloop\fi}
\def\ddef#1{\expandafter\def\csname bb#1\endcsname{\ensuremath{\mathbb{#1}}}}
\ddefloop ABCDEFGHIJKLMNOPQRSTUVWXYZ\ddefloop

\def\ddefloop#1{\ifx\ddefloop#1\else\ddef{#1}\expandafter\ddefloop\fi}
\def\ddef#1{\expandafter\def\csname fr#1\endcsname{\ensuremath{\mathfrak{#1}}}}
\ddefloop ABCDEFGHIJKLMNOPQRSTUVWXYZ\ddefloop
\def\ddefloop#1{\ifx\ddefloop#1\else\ddef{#1}\expandafter\ddefloop\fi}
\def\ddef#1{\expandafter\def\csname scr#1\endcsname{\ensuremath{\mathscr{#1}}}}
\ddefloop ABCDEFGHIJKLMNOPQRSTUVWXYZ\ddefloop
\def\ddefloop#1{\ifx\ddefloop#1\else\ddef{#1}\expandafter\ddefloop\fi}
\def\ddef#1{\expandafter\def\csname b#1\endcsname{\ensuremath{\mathbf{#1}}}}
\ddefloop ABCDEFGHIJKLMNOPQRSTUVWXYZ\ddefloop
\def\ddef#1{\expandafter\def\csname c#1\endcsname{\ensuremath{\mathcal{#1}}}}
\ddefloop ABCDEFGHIJKLMNOPQRSTUVWXYZ\ddefloop
\def\ddef#1{\expandafter\def\csname h#1\endcsname{\ensuremath{\widehat{#1}}}}
\ddefloop ABCDEFGHIJKLMNOPQRSTUVWXYZ\ddefloop
\def\ddef#1{\expandafter\def\csname t#1\endcsname{\ensuremath{\widetilde{#1}}}}
\ddefloop ABCDEFGHIJKLMNOPQRSTUVWXYZ\ddefloop

\def\ddefloop#1{\ifx\ddefloop#1\else\ddef{#1}\expandafter\ddefloop\fi}
\def\ddef#1{\expandafter\def\csname mat#1\endcsname{\ensuremath{\mathbf{#1}}}}
\ddefloop abcdefgijklmnopqrstuvwxyzABCDEFGHIJKLMNOPQRSTUVWXYZ\ddefloop

\newcommand{\Vst}{V^\star}
\newcommand{\Vpi}{V^\pi}

\newcommand{\Qst}{Q^\star}

\newcommand{\pist}{\pi^\star}
\newcommand{\pihat}{\widehat{\pi}}

\newcommand{\what}{\widehat{w}}

\newcommand{\Qpi}{Q^\pi}

\newcommand{\simplex}{\bigtriangleup}

\newcommand{\iotaeps}{\iota_\epsilon}

\newcommand{\cOtil}{\widetilde{\cO}}

\newcommand{\Vpihat}{V^{\pihat}}

\renewcommand{\ast}{a^\star}

\newcommand{\cXtil}{\widetilde{\cX}}

\newcommand{\betatil}{\widetilde{\beta}}

\newcommand{\algname}{\textsc{Force}\xspace}

\newcommand{\inner}[2]{\left \langle #1, #2 \right \rangle}
\newcommand{\innerb}[2]{\langle #1, #2 \rangle}

\newcommand{\bx}{\bm{x}}

\newcommand{\Ball}{\mathcal{B}}

\newcommand{\calN}{\mathcal{N}}

\newcommand{\covnum}{\mathsf{N}}

\newcommand{\bLambda}{\bm{\Lambda}}
\newcommand{\bv}{\bm{v}}
\newcommand{\bphi}{\bm{\phi}}

\newcommand{\btheta}{\bm{\theta}}

\newcommand{\bw}{\bm{w}}

\newcommand{\Fclass}{\mathscr{F}}

\renewcommand{\what}{\widehat{\bm{w}}}

\newcommand{\bthetast}{\btheta_{\star}}
\newcommand{\bthetahat}{\widehat{\btheta}}
\newcommand{\bu}{\bm{u}}

\newcommand{\bmu}{\bm{\mu}}

\newcommand{\Rclass}{\mathscr{R}}
\newcommand{\fes}{\textsc{Egs}\xspace}
\newcommand{\cdalg}{\textsc{CoverTraj}\xspace}
\newcommand{\cEw}{\cE_{Q}}

\newcommand{\Kmax}{K_{\mathrm{max}}}

\newcommand{\unif}{\mathrm{unif}}
\newcommand{\pacalg}{\textsc{RFLin-Plan}\xspace}
\newcommand{\rfexp}{\textsc{RFLin-Explore}\xspace}
\newcommand{\bphihat}{\widehat{\bphi}}
\newcommand{\bphist}{\bphi^\star}

\newcommand{\be}{\bm{e}}
\newcommand{\sbar}{\bar{s}}

\newcommand{\bphitil}{\widetilde{\bphi}}

\newcommand{\bern}{\mathrm{Bernoulli}}
\newcommand{\bgamma}{\bm{\gamma}}
\newcommand{\regmin}{\textsc{RegMin}\xspace}
\newcommand{\piexp}{\pi_{\mathrm{exp}}}

\def\Ebb{\mathbb{E}}

\newcommand{\E}{\Ebb}

\newcommand{\field}[1]{\mathbb{#1}}

\newcommand{\fB}{\field{B}}
\newcommand{\fS}{\field{S}}

\newcommand{\calD}{{\mathcal{D}}}

\newcommand{\one}{\boldsymbol{1}}

\usepackage{wrapfig}

\newlength\tindent
\newcommand{\rev}[1]{#1}

\allowdisplaybreaks

\makeatletter
\providecommand\theHALG@line{\thealgorithm.\arabic{ALG@line}}
\makeatother

\title{Reward-Free RL is No Harder Than Reward-Aware RL in Linear Markov Decision Processes}

\author{Andrew Wagenmaker\footnote{University of Washington, Seattle. Email: \texttt{ajwagen@cs.washington.edu}} \and Yifang Chen\footnote{University of Washington, Seattle. Email: \texttt{yifangc@cs.washington.edu}} \and Max Simchowitz\footnote{CSAIL, MIT. Email: \texttt{msimchow@mit.edu}} \and Simon S. Du\footnote{University of Washington, Seattle. Email: \texttt{ssdu@cs.washington.edu}} \and Kevin Jamieson\footnote{University of Washington, Seattle. Email: \texttt{jamieson@cs.washington.edu}}}

\date{June 18, 2022}

\begin{document}

\maketitle

\begin{abstract}

Reward-free reinforcement learning (RL) considers the setting where the agent does not have access to a reward function during exploration, but must propose a near-optimal policy for an arbitrary reward function revealed only after exploring. In the the tabular setting, it is well known that this is a more difficult problem than reward-aware (PAC) RL---where the agent has access to the reward function during exploration---with optimal sample complexities in the two settings differing by a factor of $|\mathcal{S}|$, the size of the state space. We show that this separation does not exist in the setting of linear MDPs. We first develop a computationally efficient algorithm for reward-free RL in a $d$-dimensional linear MDP with sample complexity scaling as $\widetilde{\mathcal{O}}(d^2 H^5/\epsilon^2)$. We then show a lower bound with matching dimension-dependence of $\Omega(d^2 H^2/\epsilon^2)$, which holds for the reward-aware RL setting. To our knowledge, our approach is the first computationally efficient algorithm to achieve optimal $d$ dependence in linear MDPs, even in the single-reward PAC setting. Our algorithm relies on a novel procedure which efficiently traverses a linear MDP, collecting samples in any given ``feature direction'', and enjoys a sample complexity scaling optimally in the (linear MDP equivalent of the) maximal state visitation probability. We show that this exploration procedure can also be applied to solve the problem of obtaining ``well-conditioned'' covariates in linear MDPs.

\end{abstract}


\section{Introduction}
Efficiently exploring unknown, stochastic environments is a central challenge in online reinforcement learning (RL). Whether attempting to learn a near-optimal policy or achieving large online reward, virtually all provably efficient algorithms rely on some form of exploration to guarantee the state space of the Markov Decision Process (MDP) is traversed. Indeed, failure to guarantee such exploration can result in very suboptimal performance, as near-optimal actions may be under-explored, leading to missed reward.

The \emph{reward-free} RL setting highlights the role of exploration in RL by tasking the learner with exploring their environment without access to a reward function, then revealing a reward function, and asking the learner to produce a near-optimal policy for that reward function. In a sense, to solve this problem the learner must traverse the \emph{entire} MDP, as they do not know which states and actions will lead to high reward under the yet-to-be-revealed reward function. In contrast, the \emph{reward-aware} RL setting (also known as the PAC RL setting) gives the learner access to the reward function from the beginning. The learner must again explore the MDP so as to learn a near-optimal policy, but now they may direct their exploration to focus, for instance, on the regions of the environment with large reward.

In the setting of tabular MDPs, MDPs with a finite number of states and actions, it is well known that the reward-free problem is more difficult than the reward-aware problem. Indeed, for an MDP with $|\cS|$ states and $|\cA|$ actions, to find an $\epsilon$-optimal policy, the sample complexity in the reward-free setting is known to scale as \rev{(ignoring $H$ factors)} $\Theta(|\cS|^2 |\cA|/\epsilon^2)$, while in the reward-aware setting it scales as $\Theta(|\cS| |\cA|/\epsilon^2)$, a difference of $|\cS|$. 
Intuitively, this reduction in complexity results from the above observation: when given access to a reward function, the learner need not learn every transition to equal precision, but can focus on the most relevant ones, those leading to high reward.

Recently, the RL community has turned its attention to MDPs with large state-spaces, showing that efficient learning is possible with function approximation techniques. While much progress has been made, fundamental questions remain: the optimal rates are not known for either the reward-free RL or reward-aware RL  problems, and it is also not known whether the rates exhibit a gap similar to the one in the tabular setting.   

In this work we study \emph{linear function approximation}, in particular the linear MDP setting, and show that\rev{, up to $H$ factors,} \textbf{\emph{reward-free RL is no harder than reward-aware RL in linear MDPs}}. We develop a computationally efficient reward-free algorithm which, in a $d$-dimensional linear MDP, is able to learn an $\epsilon$-optimal policy for an arbitrary number of reward function using only $\cOtil(d^2 H^5/\epsilon^2)$ samples. We then show a lower bound on reward-aware RL of $\Omega(d^2 H^2/\epsilon^2)$. Our results imply that in RL with function approximation, there is no advantage to ``directing'' exploration given access to a reward function: in the worst case, all transitions must still be explored. Furthermore, to our knowledge, this is the first result that settles the optimal $d$ dependence attainable by a computationally efficient algorithm in a linear MDP in any problem setting (reward-free, reward-aware, or regret minimization).

Our results critically rely on a novel exploration strategy able to generate ``covering trajectories''. In particular, our procedure is able to traverse the MDP and collect data in each feature direction up to a given, desired tolerance. Critically, the complexity of this procedure scales with the inverse of the ``set visitation'' probability, rather than the inverse squared, which proves essential in obtaining the optimal sample complexity. As a corollary of this approach, we also show how to collect well-conditioned covariates: covariates with lower-bounded minimum eigenvalue.

\section{Related Work}
We highlight three directions in RL that our work relates to.

\paragraph{Exploration in Reinforcement Learning.}
Arguably the most common approach to exploration in RL is that of \emph{optimism}, which is used by a host of works \citep{azar2017minimax,jin2018q,zanette2019tighter,jin2020provably,zhou2020nearly,zhang2020reinforcement}, and is typically employed to balance the exploration-exploitation tradeoff and achieve low regret. The \emph{reward-free} RL setting seeks to explore an MDP without access to a reward function, in order to then determine a near-optimal policy for an arbitrary reward function. This has been studied in the tabular setting \citep{jin2020reward,menard2020fast,zhang2020nearly,wu2021gap}, where the optimal scaling is known to be $\Theta(|\cS|^2 |\cA|/\epsilon^2)$ \citep{jin2020reward}, as well as the function approximation setting \citep{zanette2020provably,wang2020reward,zhang2021reward}. 
In the linear MDP setting, \citet{wang2020reward} show a sample complexity of $\cO(\frac{d^3 H^6}{\epsilon^2})$, and \citet{zanette2020provably} show a complexity of $\cO(\frac{d^3 H^5}{\epsilon^2})$, yet neither provides lower bounds. A related line of work on ``reward-free'' RL in linear MDPs seeks to learn a good feature representation \citep{agarwal2020flambe,modi2021model}. The exploration strategy we employ is related to that proposed in \citet{zhang2020nearly} and extended by \citet{wagenmaker2021beyond}, though both of these works consider the tabular setting. A final work of note is \citet{tarbouriech2020provably}, which seeks to collect an arbitrary desired number of samples from each state in a tabular setting.

\paragraph{Reinforcement Learning with Function Approximation.}
A subject of much recent interest in the RL community is that of RL with function approximation. The majority of attention has been devoted to MDPs with linear structure \citep{yang2019sample,jin2020provably,wang2019optimism,du2019good,zanette2020frequentist,zanette2020learning,ayoub2020model,jia2020model,modi2020sample,weisz2021exponential,zhou2020nearly,zhou2021provably,zhang2021variance,wang2021exponential,wagenmaker2021first,huang2022towards}. Several different settings of MDPs with linear structure have been proposed. Most common are the linear MDP assumption \citep{jin2020provably}, which is the setting we consider in this work, and the linear mixture MDP assumption \citep{jia2020model,ayoub2020model,zhou2020nearly}. A second line of work seeks to determine more general conditions that allow for efficient learning in large state-space MDPs and general function approximation techniques \citep{jiang2017contextual,du2021bilinear,jin2021bellman,foster2021statistical}.

\paragraph{Reward-Aware Reinforcement Learning.}
Reward-aware RL (also known as PAC RL) has a long history in reinforcement learning, in particular in the tabular setting \citep{kearns2002near,kakade2003sample,dann2015sample,dann2017unifying}. The current state of the art in tabular PAC RL, which achieves an optimal scaling of $\cO(|\cS||\cA|/\epsilon^2)$, is \citep{dann2019policy,menard2020fast}. 
In the linear MDP setting, the literature has tended to focus on achieving low regret. However, any low-regret algorithm can be used to obtain an $\epsilon$-optimal policy, thereby solving the PAC problem, via an \emph{online-to-batch conversion} \citep{jin2018q,menard2020fast}. Existing low-regret algorithms in linear MDPs \citep{jin2020provably,zanette2020learning,wagenmaker2021first} can thus be used to solve PAC RL. However, computationally efficient procedures \citep{jin2020provably,wagenmaker2021first} at best achieve a sample complexity of $\cO(d^3 \cdot \poly(H)/\epsilon^2)$. While \citep{zanette2020learning,jin2021bellman} achieve complexities of $\cO(d^2 \cdot \poly(H)/\epsilon^2)$, their approaches are computationally inefficient. 
Note that it is not straightforward to apply these algorithms to the reward-free setting---they rely on access to a reward function during exploration. To our knowledge no lower bounds exist for PAC RL in linear MDPs. As such, it remains an open question what the optimal $d$ dependence is, and if it can be obtained by a computationally efficient algorithm.


\section{Preliminaries}

\newcommand{\roundtwo}{\mathsf{rnd}_2}
\paragraph{Notation.}
All logarithms $\log$ are base-$e$ unless otherwise noted. We let $[m] = \{ 1,2,\ldots,m \}$, $\Ball^d(R) := \{\bx \in \R^d:\|\bx\|\le R\}$, the ball of radius $R$ in $\R^d$, and specialize $\Ball^d := \Ball^d(1)$ to the unit ball. $\cS^{d-1}$ denotes the unit sphere in $\R^d$. 
$\cO(\cdot)$  hides absolute constants, and $\cOtil(\cdot)$  hides absolute constants and logarithmic terms. Throughout, bold characters refer to vectors and matrices and standard characters refer to scalars.

\subsection{Markov Decision Processes}
We study finite-horizon, episodic Markov Decision Processes with time-varying transition kernel. We denote an MDP by the tuple $\cM = (\cS,\cA,H,\{P_h \}_{h=1}^H, \{ r_h \}_{h=1}^H)$, where $\cS$ is the set of states, $\cA$ the set of actions, $H$ the horizon, $\{P_h \}_{h=1}^H$ the transition kernel, $P_h : \cS \times \cA \rightarrow \simplex(\cS)$, and $\{ r_h \}_{h=1}^H$ the reward, $r_h : \cS \times \cA \rightarrow [0,1]$, which we assume is deterministic\footnote{The assumption that the rewards are deterministic is for expositional convenience only---all our results could easily be modified to handle random rewards.}. We do not include an initial state distribution but instead assume the MDP always starts in state $s_1$ and encode the initial state distribution in the first transition, which is without loss of generality. We assume that $\{P_h \}_{h=1}^H$ is initially unknown.

A policy, $\pi = \{ \pi_h \}_{h=1}^H$, $\pi_h : \cS \rightarrow \simplex(\cA)$, is a mapping from states to actions. Given a policy $\pi$, an episode starts in state $s_1$, the agent takes action $a_1 \sim \pi_1(s_1)$, the MDP transitions to state $s_2 \sim P_1(\cdot | s_1,a_1)$, and the agent receives (and observes) reward $r_1(s_1,a_1)$. This process repeats for $H$ steps: at step $h$, if the agent is in state $s_h$, they take action $a_h \sim \pi_h(s_h)$, transition to state $s_{h+1} \sim P_h(\cdot | s_h, a_h)$, and receive reward $r_h(s_h,a_h)$. After $H$ steps the episode terminates and restarts at $s_1$. In the special case when the policy is deterministic---$\pi_h$ is supported on only one action for each $s$ and $h$---we will denote this action as $\pi_h(s)$. 

We let $\Exp_h[V](s,a) = \Exp_{s' \sim P_h(\cdot | s,a)}[V(s')]$, so $\Exp_h[V](s,a)$ denotes the expected next-state value of $V$ given from state $s$ after playing action $a$ at time $h$. For a fixed policy $\pi$, we let $\Exp_{\pi}[\cdot]$ denote the expectation over the trajectory $(s_1,a_1,\ldots,s_H,a_H)$ induced by $\pi$ on the MDP.

\paragraph{Value Functions.}
For a policy $\pi$, the $Q$-value function, $\Qpi_h : \cS \times \cA \rightarrow [0,H]$ is defined as the expected reward one would obtain if action $a$ is taken in state $s$ at step $h$ and then $\pi$ is followed for all subsequent steps. Precisely, 
\begin{align*}
\Qpi_h(s,a) = \Exp_\pi \left [ \sum_{h' = h}^H r_{h'}(s_{h'},a_{h'}) | s_h = s, a_h = a \right ] .
\end{align*}
The value function, $\Vpi_h : \cS \rightarrow [0,H]$, is defined as the expected reward one would obtain if policy $\pi$ is played from state $s$ at step $h$ onwards. In terms of the $Q$-value function, $\Vpi_h(s) = \Exp_{a \sim \pi_h(s)}[\Qpi_h(s,a)]$. The $Q$-value function and value function can be related by the \emph{Bellman equation}:
\begin{align*}
\Qpi_h(s,a) = r_h(s,a) + \Exp_h[\Vpi_{h+1}](s,a) .
\end{align*}
For simplicity and to ensure this relationship holds for all $h \in [H]$, we define $\Vpi_{H+1}(s) = \Qpi_{H+1}(s,a) := 0$ for all $s,a$, and $\pi$. We denote the \emph{value} of a policy by $\Vpi_0 := \Exp_{\pi}[\Vpi_1(s_1)]$, the expected reward $\pi$ obtains in an episode.
The optimal policy, $\pist$, is the policy which achieves the largest expected reward: $V^{\pist}_0 = \sup_\pi \Vpi_0$ (note that $\pist$ need not be unique). We will denote the value function and $Q$-value function associated with $\pist$ as $\Vst_h(s)$ and $\Qst_h(s,a)$, respectively. Note that $\Vst$ and $\Qst$ satisfy $\Vst_h(s) = \sup_\pi \Vpi_h(s)$ and $\Qst_h(s,a) = \sup_\pi \Qpi_h(s,a)$. 

The value function depends on both the MDP and the reward function. In cases where we want to make this dependence explicit, we denote $\Vpi_0(r)$ the value of policy $\pi$ \emph{with respect to reward function $r$}, and similarly define $\Vst_0(r)$ as the value of the optimal policy with respect to $r$.

\paragraph{Reward-Aware RL (PAC Policy Identification).}
In the reward-aware RL setting, given a reward function $r$, the goal is to identify a policy $\pihat$ such that, with probability at least $1-\delta$,
\begin{align}\label{eq:pac_defn}
\Vst_0(r) - \Vpihat_0(r) \le \epsilon
\end{align}
using as few episodes as possible. We call a policy $\pihat$ satisfying \eqref{eq:pac_defn} \emph{$\epsilon$-optimal}. Critically, in reward-aware RL the learner has access to rewards while exploring. Note that this setting is also commonly referred to as the \emph{PAC RL} (Probably Approximately Correct) setting in the literature.

\paragraph{Reward Free RL.}
The reward-free RL setting stands in contrast to the reward-aware RL setting in that the learner does not observe the rewards (or have any access to the reward function) while exploring.
The reward-free RL setting proceeds in two phases:
\begin{enumerate}
\item Given $(\epsilon,\delta)$, without observing any rewards, the learner explores an MDP for $K$ episodes, where $K$ is of the learner's choosing, collecting trajectories $\{ (s_{1,k}, a_{1,k}, s_{2,k}, a_{2,k}, \ldots, s_{H,k}, a_{H,k}) \}_{k=1}^K$. 
\item The learner outputs a map $\pihat(\cdot)$ which takes as input a reward function and outputs a policy such that $\Vst_0(r) - V_0^{\pihat(r)}(r) \le \epsilon$ for all valid reward functions $r$ simultaneously.
\end{enumerate}
As in reward-aware RL, the goal is to obtain a procedure able to provide the above guarantee with probability at least $1-\delta$, and do so using as few episodes as possible.

\subsection{Reinforcement Learning with Linear Function Approximation}
The RL theory literature has often considered the \emph{tabular} setting, where $|\cS| < \infty$, $|\cA| < \infty$. While convenient to work with theoretically, this setting is quite limited in practice, and is not a realistic model of real world settings where the state spaces may be infinite, and ``nearby'' states may behave similarly. Recently, the RL theory community has begun relaxing the tabular assumption and studying large state-space settings. In this work, we consider the linear MDP setting of \cite{jin2020provably}. Linear MDPs are defined as follows.

\begin{defn}[Linear MDPs]\label{defn:linear_mdp}
We say that an MDP is a $d$-\emph{dimensional linear MDP}, if there exists some (known) feature map $\bphi(s,a) : \cS \times \cA \rightarrow \R^d$, $H$ (unknown) signed vector-valued measures $\bmu_h \in \R^d$ over $\cS$, and $H$ (unknown) reward vectors $\btheta_h \in \R^d$, such that:
\begin{align*}
P_h(\cdot | s,a) = \innerb{\bphi(s,a)}{\bmu_h(\cdot)}, \quad r_h(s,a) = \innerb{\bphi(s,a)}{\btheta_h}. 
\end{align*}
We will assume $\| \bphi(s,a) \|_2 \le 1$ for all $s,a$; and for all $h$, $\| |\bmu_h|(\cS) \|_2 = \| \int_{s \in \cS} | \rmd \bmu_h(s) | \|_2  \le \sqrt{d}$ and $\| \btheta_h \|_2 \le \sqrt{d}$.
\end{defn}

\citet{jin2020provably} show that the linear MDP setting encompasses tabular MDPs (with $d = | \cS | | \cA |$). Critically, however, it also encompasses MDPs with infinite state-spaces, for instance, MDPs where the feature space is the $d$-dimensional simplex. In addition, the linear MDP assumptions allows for ``nearby'' states, states with similar feature vectors, to behave in similar ways, allowing us to generalize across states without visiting every state.

Finally, we introduce the concept of set visitations, which will play an important role in our analysis. 

\begin{defn}[Set Visitation]
For any $\cX \subseteq \R^d$ and policy $\pi$, we let
\begin{align*}
\omega_h^\pi(\cX) := \bbP_\pi[\bphi(s_h,a_h) \in \cX] 
\end{align*}
denote the probability of visiting $\cX$ under $\pi$ at step $h$. Furthermore, we will say that $\cX$ is \emph{$c$-unreachable} if $\sup_{\pi} w_h^\pi(\cX) \le c $.
\end{defn}


\section{Main Results} 
We first present our main upper bound on reward-free RL, and then provide a lower bound on reward-aware RL---giving a lower bound on reward-free RL as an immediate corollary.

\subsection{Upper Bounds on Reward Free RL in Linear MDPs}

We present our main algorithm, \rfexp, as \Cref{alg:rf_exp}. 

\iftoggle{arxiv}{}{\algrenewcommand\algorithmicindent{0.8em}}
\begin{algorithm}[h]
\caption{Reward Free Linear RL: Exploration (\rfexp)}
\begin{algorithmic}[1]
\State \textbf{input:} tolerance $\epsilon$, confidence $\delta$
\State $\Kmax \leftarrow \poly(1/\epsilon, d, H, \log 1/\delta)$
\State $\beta \leftarrow c H \sqrt{d \log(1 + dH\Kmax) + \log H/\delta}$
\State $\iotaeps \leftarrow \lceil \log_2(4 \beta H/\epsilon) \rceil$ 
\State $\bgamma^2 \in \R^{\iotaeps}$ with $\bgamma_i^2 \leftarrow 2^{2i}  \cdot \frac{\epsilon^2}{64 H^2 \iotaeps^2 \beta^2}$ for $i \in [\iotaeps]$.
\For{$h = 1,\ldots,H$}
	\State $\{ (\cX_{hi},\cD_{hi}, \bLambda_{hi}) \}_{i=1}^{\iotaeps} \leftarrow \cdalg(h,\tfrac{\delta}{H},\bgamma^2,\iotaeps)$ 
\EndFor
\State $\mathfrak{D} \leftarrow \{ \{ (\cX_{hi},\cD_{hi}, \bLambda_{hi}) \}_{i=1}^{\iotaeps} \}_{h=1}^H$ \\
\Return $\pacalg(\cdot \ ; \epsilon, \delta, \mathfrak{D})$
\end{algorithmic}
\label{alg:rf_exp}
\end{algorithm}

\rfexp relies on a core subroutine, \cdalg, to collect data. We provide a detailed description of \cdalg in \Cref{sec:cover} but give a brief description below. 

\paragraph{\cdalg Subroutine.}
\cdalg generates a set of \emph{covering trajectories}. For a given $h$, a call to \cdalg partitions the feature space into sets $\cX_{hi}$ such that
\begin{align*}
\sup_\pi w_h^\pi(\cX_{hi}) \le 2^{-i+1}
\end{align*}
and collects data $\cD_{hi} = \{ (s_{h,\tau}^i, a_{h,\tau}^i, s_{h+1,\tau}^i) \}_{\tau = 1}^{K_i}$ and covariates $\bLambda_{hi} = I + \sum_{\tau = 1}^{K_i} \bphi_{h,\tau}^i (\bphi_{h,\tau}^i)^\top$ for $\bphi_{h,\tau}^i := \bphi(s_{h,\tau}^i, a_{h,\tau}^i)$ such that 
\begin{align*}
\bphi^\top \bLambda_{hi}^{-1} \bphi \le \bgamma_i^2, \quad \forall \bphi \in \cX_{hi} .
\end{align*}
In words, each $\cX_{hi}$ corresponds to a set of feature vectors that are $2^{-i+1}$-unreachable, and for which we can guarantee a desired amount of data pointing in similar directions has been collected. If a particular direction is difficult to reach, then it is unlikely any policy will encounter it and, as such, to find a near-optimal policy, it suffices to use a relatively large learning tolerance in that direction. Motivated by this, we set $\bgamma_i = \cO(2^{i} \epsilon)$---as the sets become more difficult to reach, we require that less data is collected from them.

\paragraph{\pacalg Mapping.}
\rfexp returns $\pacalg(\cdot \ ; \epsilon, \delta, \mathfrak{D})$ which defines a mapping from reward functions to policies. This mapping is parameterized by the data returned by \cdalg and, given a reward function as input, outputs a policy it believes is near-optimal for this reward function. \pacalg itself runs a simple least squares value iteration procedure to compute the policy. We define \pacalg in \Cref{alg:pac}. We then have the following result.

\iftoggle{arxiv}{}{\algrenewcommand\algorithmicindent{1.5em}}
\begin{algorithm}[h]
\caption{Reward Free Linear RL: Planning (\pacalg)}
\begin{algorithmic}[1]
\State \textbf{input:} reward functions $r$
\State \textbf{parameters:} tolerance $\epsilon$, confidence $\delta$, data $\{ \{ (\cX_{hi},\cD_{hi}, \bLambda_{hi}) \}_{i=1}^{\iotaeps} \}_{h=1}^H$
\State $\Kmax \leftarrow \poly(1/\epsilon, d, H, \log 1/\delta)$
\State $\beta \leftarrow c H \sqrt{d \log(1 + d H \Kmax) + \log H/\delta}$
\For{$h = H,H-1,\ldots,1$}
\iftoggle{arxiv}{
\State $\what_h \leftarrow \argmin_{\bw} \sum_{i=1}^{\iotaeps} \sum_{\tau = 1}^{K_i} \big ( \bw^\top \bphi_{h,\tau}^i  - r_{h}(s_{h,\tau}^i,a_{h,\tau}^i)  - V_{h+1}(s_{h+1,\tau}^i) \big )^2 +  \| \bw \|_2^2$
}
{
\State $\what_h \leftarrow \argmin_{\bw} \sum_{i=1}^{\iotaeps} \sum_{\tau = 1}^{K_i} \big ( \bw^\top \bphi_{h,\tau}^i $ 
\Statex \hspace{4em} $- r_{h}(s_{h,\tau}^i,a_{h,\tau}^i)  - V_{h+1}(s_{h+1,\tau}^i) \big )^2 +  \| \bw \|_2^2$
}
\State $\bLambda_h =  I  + \sum_{i=1}^{\iotaeps} \sum_{\tau = 1}^{K_i} \bphi_{h,\tau}^i (\bphi_{h,\tau}^i)^\top$
\State $Q_h(\cdot,\cdot) \leftarrow \min \{ \innerb{\bphi(\cdot,\cdot)}{\what_h} + \beta \| \bphi(\cdot,\cdot) \|_{\bLambda_h^{-1}}, H \}$
\State $V_h(\cdot) \leftarrow \max_a Q_h(\cdot,a)$
\State $\pihat_h(\cdot) \leftarrow \argmax_{a} Q_h(\cdot,a)$
\EndFor
\Return $\{ \pihat_h \}_{h=1}^H$
\end{algorithmic}
\label{alg:pac}
\end{algorithm}


\begin{thm}\label{thm:pac}
Consider running \rfexp with tolerance $\epsilon > 0$ and confidence $\delta > 0$, and let $\pacalg(\cdot)$ denote its output. Then with probability at least $1-\delta$, for an arbitrary number of reward functions $r$ satisfying \Cref{defn:linear_mdp}, $\pacalg(r)$ returns a policy $\pihat$ which is $\epsilon$-optimal with respect to $r$. Furthermore, this procedure collects at most
\begin{align*}
\cOtil \left ( \frac{d H^5 (d + \log 1/\delta) }{\epsilon^2} + \frac{d^{9/2} H^6 \log^{4}(1/\delta)}{\epsilon} \right ) 
\end{align*}
episodes, where $\cOtil(\cdot)$ hides absolute constants and terms $\poly\log(d,H,1/\epsilon)$ and $\poly \log \log( 1/\delta)$. 
\end{thm}
For $\epsilon$ sufficiently small, exploring via \rfexp and planning via \pacalg learns an $\epsilon$-optimal policy for an arbitrarily large number of reward functions using only $\cOtil(\frac{d^2 H^5}{\epsilon^2})$ episodes. This improves on both existing reward-free algorithms for linear MDPs by a factor of $d$ \citep{zanette2020provably,wang2020reward}.

\begin{rem}[Computational Efficiency]
In MDPs with a finite number of actions, both \rfexp and \pacalg are computationally efficient, with computational cost scaling polynomially in $d,H,1/\epsilon,\log 1/\delta$, and $|\cA|$. The primary computational cost of \rfexp is due to \cdalg, while the computational cost of \pacalg is due primarily to solving a least-squares problem. The computational cost of \cdalg is described in more detail in \Cref{sec:cover}, but is dominated by calls to a computationally efficient regret minimization algorithm.
In the case when an infinite number of actions is available, the $|\cA|$ dependence can be replaced with $2^{\cO(d)}$. As several existing works have noted, this dependence seems unavoidable \citep{jin2020provably,wagenmaker2021first}. 
\end{rem}

\begin{rem}[Reward-Aware RL]

When only a single reward function is given, we recover the standard reward-aware RL setting. Hence,  \rfexp and \pacalg provide a computationally efficient reward-aware RL algorithm that learns an $\epsilon$-optimal policy after collecting $\cOtil(\frac{d H^5 ( d + \log 1/\delta)}{\epsilon^2})$ episodes. Our more general \emph{reward-free} result is sharper than the less-general reward-aware results derived from an online-to-batch conversion of prior computationally efficient low-regret  algorithms \citep{jin2020provably,wagenmaker2021first} by a factor of $d$. Our result also improves on the computationally-inefficient $\cO(\frac{d^2 H^4}{\epsilon^2})$ reward-aware bound due to \cite{zanette2020learning} in that (a) it is computationally efficient, and (b) it extends to the reward-free setting. 
\end{rem}

\begin{rem}[Nonlinear Rewards]
Our procedure is able to handle nonlinear reward functions with little modification. The primary difference in the setting of nonlinear rewards is that we must have query access to the reward function for \emph{all} $s,a,h$. In contrast, if the reward functions are linear, our procedure only needs to know the reward function values for the trajectories observed during exploration. 
\end{rem}


\subsection{Lower Bounds on Reward-Aware RL in Linear MDPs}
To our knowledge, \Cref{thm:pac} is the first result to show that computationally efficient $d^2$ complexity is possible in linear MDPs for any of the reward-free, reward-aware, or regret minimization settings. We next show a somewhat surprising result: reward-free RL is no harder than reward-aware RL in linear MDPs, up to horizon factors. To this end, we show a lower bound on reward-aware RL in linear MDPs. As a warm-up, we first provide a lower bound for the simpler linear bandit setting.

\begin{defn}[Linear Bandits and $\epsilon$-Optimal Arms]\label{def:lin_bandit}
Consider the linear bandit setting parameterized by some $\btheta \in \R^d$ and $\Phi \subseteq \R^d$, where at every step $k$ the learner chooses $\bphi_k \in \Phi$ and observes
\begin{align*}
y_k \sim \bern(\innerb{\btheta}{\bphi_k} + 1/2).
\end{align*}
We say an arm $\bphi \in \Phi$ is $\epsilon$-optimal if
\begin{align*}
\innerb{\btheta}{\bphi} + 1/2 \ge \sup_{\bphi' \in \Phi} \innerb{\btheta}{\bphi'} + 1/2 - \epsilon 
\end{align*}
and a policy $\pi \in \simplex_{\Phi}$ is $\epsilon$-optimal if
\begin{align*}
\Exp_{\bphi \sim \pi}[\innerb{\btheta}{\bphi}] + 1/2 \ge \sup_{\bphi' \in \Phi} \innerb{\btheta}{\bphi'} + 1/2 - \epsilon .
\end{align*}
\end{defn}

\begin{thm}\label{thm:pac_lb}
Fix $\epsilon > 0$, $d > 1$, and $K \ge d^2$. Let $\btheta \in \Theta = \{ -\sqrt{d/700K}, \sqrt{d/700K} \}^d$ and $\Phi = \cS^{d-1}$. Consider running a (possibly adaptive) algorithm in the linear bandit setting of \Cref{def:lin_bandit} which stops at (a possibly random stopping time) $\tau$ and outputs a guess at an $\epsilon$-optimal policy, $\pihat \in \simplex_{\Phi}$. Let $\cE$ be the event
\begin{align*}
\cE := \{ \tau \le K \text{ and } \pihat \text{ is $\epsilon$-optimal} \} .
\end{align*}
Then unless $K \ge  c \cdot \frac{d^2}{\epsilon^2}$ for a universal $c > 0$, there exists $\btheta \in \Theta$ for which $\Pr_{\btheta}[\cE^c] \ge 1/10$; i.e., with constant probability either $\pihat$ is not $\epsilon$-optimal or more than $K$ samples are collected. 
\end{thm}

Setting $K = \frac{c}{2} \cdot \frac{d^2}{\epsilon^2}$, \Cref{thm:pac_lb} implies that with constant probability, any algorithm will gather either more than $\frac{c}{2} \cdot \frac{d^2}{\epsilon^2}$ samples, or will output a policy which is not $\epsilon$-optimal. While this shows that $\Omega(d^2/\epsilon^2)$ samples is necessary to find an $\epsilon$-optimal policy in linear bandits, there is a slight discrepancy between the model class in \Cref{def:lin_bandit} and the linear MDP setting, \Cref{defn:linear_mdp}. In the former, the rewards are assumed to be random, while in the latter, the rewards are deterministic. 
Nevertheless, we show in \Cref{sec:lin_bandit_lin_mdp} that it is possible to encode the linear bandit structure of \Cref{def:lin_bandit} in a linear MDP with deterministic, known rewards.  This yields  the following corollary.

\begin{cor}\label{cor:lin_mdp_lb}
Fix $\epsilon > 0$, $d > 1$, $H > 1$, and $K \ge d^2$. Consider running a (possibly adaptive) algorithm for $K$ episodes in a $(d+1)$-dimensional linear MDP with horizon $H$, which stops at (a possibly random stopping time) $\tau$ and outputs a guess at an $\epsilon$-optimal policy, $\pihat$. Let $\cE$ be the event
\begin{align*}
\cE := \{ \tau \le K \text{ and } \pihat \text{ is $\epsilon$-optimal} \} .
\end{align*} 
Then  there is a universal constant $c > 0$ such that unless
\begin{align*}
K \ge c \cdot \frac{d^2 H^2}{\epsilon^2},
\end{align*}
there exists a linear MDP $\cM$ for which $\Pr_{\cM}[\cE^c] \ge 1/10$. 
\end{cor}

As \Cref{cor:lin_mdp_lb} shows, a $\Omega(d^2 H^2/\epsilon^2)$ dependence is necessary for learning an $\epsilon$-optimal policy in linear MDPs. Note that this lower bound holds for learning a \emph{single} policy given access to the reward function during exploration (indeed, it holds in the case when the learner has oracle access to the rewards)---the reward-aware RL setting. 

In contrast, \Cref{thm:pac} shows that we can learn $\epsilon$-optimal policies for \emph{all valid reward functions simultaneously}, without having access to the reward functions during exploration, using only $\cOtil(d^2 H^5/\epsilon^2)$ episodes. In other words, \rev{up to $H$ factors}, \textbf{\emph{reward-free RL is no harder than reward-aware RL in linear MDPs}}. This is in contrast to the tabular setting, where it is well known that the optimal rate for reward-aware RL is $\Theta(|\cS||\cA|/\epsilon^2)$ while the optimal rate for reward-free RL is $\Theta(|\cS|^2 |\cA|/\epsilon^2)$.

\begin{rem}[$H$ Dependence]
\rev{Our claim that ``reward-free RL is no harder than reward-aware RL in linear MDPs'' only holds up to $H$ factors---as \Cref{thm:pac} and \Cref{cor:lin_mdp_lb} show, there is still a discrepancy in the $H$ dependence for our reward-free upper bound and reward-aware lower bound. In the tabular setting, the difference in hardness between the reward-free and reward-aware setting lies in the ``dimensionality'' factors (that is, $|\cS|$)---the optimal $H$ dependence for each is identical (compare the reward-free rate of \cite{zhang2020nearly} with the reward-aware rate of \cite{zhang2020reinforcement}). Thus, our result shows that the difference in hardness between the reward-free and reward-aware problems which is present in the tabular setting is not present in the linear setting, motivating our claim.}

\rev{We do not believe the $H$ dependence of \Cref{thm:pac} is optimal, and also conjecture that the $H$ dependence in the lower bound of \Cref{cor:lin_mdp_lb} can be increased, to obtain matching $H$ dependence in our upper and lower bounds. }
As the goal of this paper is optimizing $d$ factors and not $H$ factors, we did not focus on improving the $H$ dependence, and leave this for future work. 
\end{rem}

\section{Efficient Exploration in Linear MDPs}\label{sec:cover}
Our main algorithmic technique is an exploration procedure able to efficiently generate ``covering trajectories'' in linear MDPs, which we believe may be of independent interest. Before presenting our MDP covering algorithm, \cdalg, we describe its key subroutine, \fes.

\paragraph{Explore Goal Set (\fes) Subroutine.}

\begin{algorithm}[h]
\caption{Explore Goal Set (\fes)}
\begin{algorithmic}[1]
\State \textbf{input:} goal set $\cX \subseteq \R^d$, step $h$, number of episodes $K$, collection tolerance $\gamma^2$, regret minimization algorithm \regmin
\State Run \regmin for $K$ episodes, using reward $r_h^k(s,a) := r(\bphi(s,a); \bLambda_{h,k-1})$ at episode $k$, for $r$ defined as in \eqref{eq:exp_reward}, $\bLambda_{h,k-1} =  I + \sum_{\tau = 1}^{k-1} \bphi_{h,\tau} \bphi_{h,\tau}^\top $, and $r_{h'}^k(s,a) = 0$ for $h' \neq h$
\State Collect transitions observed at step $h$ 
$$\cD \leftarrow \{ (s_{h,\tau}, a_{h,\tau}, s_{h+1,\tau}) \}_{\tau = 1}^K$$
\State \textbf{return} $\{ \bphi \in \cX \ : \ \bphi^\top \bLambda_{h,K}^{-1} \bphi \le \gamma^2 \}$, $\cD$, $\bLambda_{h,K}$
\end{algorithmic}
\label{alg:find_exp_set}
\end{algorithm}

\fes takes as input a target set $\cX$ to be explored; we fix $\cX$ in what follows. It 
then creates an ``exploration reward function''---which places a high reward on regions of $\cX$ for which we have large uncertainty---and then runs a regret minimization algorithm on this reward function to direct exploration to these regions. More specifically, we define the reward function
\begin{align}\label{eq:exp_reward}
r(\bphi;\bLambda) \leftarrow \left \{ \begin{matrix} 1 & \| \bphi \|_{\bLambda^{-1}}^2 > \gamma^2, \bphi \in \cX \\
\frac{1}{\gamma^2}  \| \bphi \|_{\bLambda^{-1}}^2  & \| \bphi \|_{\bLambda^{-1}}^2 \le \gamma^2, \bphi \in \cX \\
0 & \bphi \not\in \cX \end{matrix} \right . .
\end{align} 
At the $k$th episode, we instantiate our reward as $r_h^k(s,a) := r(\bphi(s,a); \bLambda_{h,k-1})$. This captures the fact that we have collected covariates $\bLambda_{h,k-1} = I + \sum_{\tau = 1}^{k-1} \bphi_{h,\tau} \bphi_{h,\tau}^\top$ over the first $k-1$ episodes, so our uncertainty in direction $\bphi$ scales as $\| \bphi \|_{\bLambda_{h,k-1}^{-1}}$. 
By choosing our reward function to increase as $\| \bphi \|_{\bLambda_{h,k-1}^{-1}}$ increases (for $\bphi \in \cX$), we incentive exploring directions in $\bphi \in \cX$ with large uncertainty. \fes also takes as input a scalar tolerance $\gamma^2$ and, after running for $K$ episodes, returns the set of $\bphi \in \cX$ which have been explored up to tolerance $\gamma^2$.

In order to efficiently collect data, 
we require a regret-minimization algorithm which achieves low regret with respect to a \emph{time-varying} reward function, $r^k$. We define regret with respect to such a reward function as
\begin{align*}
\cR_K := \textstyle \sum_{k=1}^K [ \Vst_0(r^k) - V_0^{\pi_k}(r^k)] .
\end{align*}
To achieve the optimal scaling in $d$, \regmin must attain \emph{first order} regret in the following sense.

\begin{defn}[First-Order Regret Minimization Algorithm]\label{asm:regret_alg}
Consider a time-varying reward function $r^k$ that is $\cF_{k-1}$-measurable, satisfies $r_h^k(s,a) \in [0,1]$, and is non-increasing in $k$ (that is, $r_h^k(s,a) \le r_h^{k-1}(s,a)$ for all $s,a,h,k$). Then we call a regret minimization algorithm \regmin a \emph{first-order regret minimization algorithm} if it achieves regret bounded as, with probability at least $1-\delta$,
\begin{align*}
\cR_K \le \sqrt{\cC_1 \Vst_0(r^1) K \cdot \log^{p_1}(HK/\delta) } + \cC_2 \log^{p_2}(HK/\delta)
\end{align*}
for constants $\cC_1, \cC_2, p_1, p_2$ which do not depend on $K$. 
\end{defn}

In general, we can think of $\cC_1$ and $\cC_2$ as $\poly(d,H)$ and $p_1$ and $p_2$ as absolute constants. We are now ready to present our main exploration algorithm, \cdalg, in \Cref{alg:find_exp_set_levels}.

\begin{algorithm}[h]
\caption{Collect Covering Trajectories (\cdalg)}
\begin{algorithmic}[1]
\State \textbf{input:} step $h$, confidence $\delta$, number of epochs $m$, tolerance vector $\bgamma^2 \in [0,1]^m$, regret minimization algorithm \regmin (default: \algname, see \Cref{sec:covtraj_force})
\State $\cX \leftarrow \cB^{d}$
\For{$i = 1,2,\ldots, m$}
	\State Set $K_i$ as in \eqref{eq:Ki_defn}
	\State $\cX_i, \cD_i, \bLambda_i \leftarrow \fes(\cX,h,K_i,\bgamma^2_i,\regmin$)
	\State $\cX \leftarrow \cX \backslash \cX_i$
\EndFor
\State \textbf{return} $\{ (\cX_i,\cD_i, \bLambda_i) \}_{i=1}^m$
\end{algorithmic}
\label{alg:find_exp_set_levels}
\end{algorithm}

\paragraph{\cdalg Algorithm.} \cdalg partitions the feature space by repeatedly calling \fes while exponentially increasing the number of episodes \fes is run for. As the number of episodes \fes is run for increases, \fes is able to reach harder and harder to reach feature directions, while ensuring easier to reach directions have been explored up to their desired tolerance $\bgamma^2_i$. Specifically, at each epoch $i$, \fes runs for 
\iftoggle{arxiv}{
\begin{align}
	K_i \leftarrow \cOtil \bigg ( 2^i & \cdot \max \bigg \{ \frac{d}{\bgamma^2_i}  \log \tfrac{2^i}{\bgamma^2_i},  \cC_1 (m p_1)^{p_1} \log^{p_1} \tfrac{1}{\delta},  \cC_2 (m p_2)^{p_2} \log^{p_2} \tfrac{1}{\delta} \bigg \} \bigg ) \label{eq:Ki_defn}
\end{align}
}
{
\begin{align}
	K_i \leftarrow \cOtil \bigg ( 2^i & \cdot \max \bigg \{ \frac{d}{\bgamma^2_i}  \log \tfrac{2^i}{\bgamma^2_i},  \cC_1 (m p_1)^{p_1} \log^{p_1} \tfrac{1}{\delta}, \nonumber \\
& \quad \cC_2 (m p_2)^{p_2} \log^{p_2} \tfrac{1}{\delta} \bigg \} \bigg ) \label{eq:Ki_defn}
\end{align}
}
episodes.  \cdalg enjoys the following guarantee.

\begin{thm}\label{thm:policy_cover}
Fix $h \in [H]$, $\delta > 0$, $m \ge 1$, and set $\bgamma^2 \in [0,1]^m$ to desired tolerances. Consider running \cdalg with a regret minimization algorithm \regmin satisfying \Cref{asm:regret_alg}, and let $\{ (\cX_i,\cD_i, \bLambda_i) \}_{i=1}^m$ denote the arguments returned. 
Then, with probability at least $1-\delta$, for each $i \in [m]$ simultaneously:
\begin{align*}
&\sup_\pi w_h^\pi(\cX_i) \le 2^{-i+1} \quad \text{and} \quad \bphi^\top \bLambda_{i}^{-1} \bphi \le \bgamma^2_i, \forall \bphi \in \cX_i
\end{align*}
and, on the same event, it also holds that
\begin{align*}
\sup_\pi w_h^\pi (\Ball^{d} \backslash \cup_{i=1}^m \cX_i) \le 2^{-m} . 
\end{align*}
Furthermore, \cdalg terminates after at most $\sum_{i=1}^m K_i$ episodes, for $K_i$ as in \eqref{eq:Ki_defn}. 
\end{thm}

\Cref{thm:policy_cover} shows that \cdalg partitions the feature space into sets $\cX_i$ such that $\cX_i$ is $2^{-i+1}$-unreachable, and where we have learned every $\bphi \in \cX_i$ up to tolerance $\bgamma^2_i$: $\| \bphi \|_{\bLambda_i^{-1}}^2 \le \bgamma^2_i$. Furthermore, it takes roughly $2^i \cdot \frac{d}{\bgamma^2_i}$ episodes of exploration to accomplish this, which is the intuitively correct rate, as the following example illustrates.

\begin{exmp}[Tabular MDPs]\label{ex:tabular}
Consider a tabular MDP with $H = 2$, state space $\cS$ with $|\cS|<\infty$, and $A$ actions (we think of $A \ll |\cS|$ as an absolute constant). Assume we always start in state $s_0$ and can break the state space into sets $\cS_1,\ldots,\cS_A$ such that $|\cS_1| = \cO(|\cS|/A)$ and where, if we take action $a_i$, we will end up in $\cS_i$ with probability $2^{-i}$ (with equal probability of being in any particular state within $\cS_i$), and $s_0$ with probability $1 - 2^{-i}$. Representing this as a linear MDP, we  have $d = A | \cS|$ and $\bphi(s,a) = \be_{sa}$, a standard basis vector. As such, $\bLambda_K$, the covariates collected over $K$ episodes, is diagonal with $[\bLambda_K]_{sa} = N(s,a)$ the number of visits to $s,a$. 

Now assume we want to ensure that $\bphi(s,a)^\top \bLambda_K^{-1} \bphi(s,a) \le \bgamma_i^2$ for some $\bgamma_i^2$, each $s \in \cS_i$, and all $a \in [A]$. For any given $s \in \cS_i$, if we take action $a_i$ in state $s_0$, we will arrive in state $s$ with probability $2^{-i}/|\cS_i|$. Thus, in expectation, to collect $N$ samples from $s$, we must run for at least
\begin{align*}
\tfrac{N}{2^{-i}/|\cS_i|} = 2^i | \cS_i | N
\end{align*} 
episodes. It follows that to collect $N$ samples from each $s \in \cS_i$ and all $a \in [A]$, it will take at least $2^i |\cS_i| A N$ episodes. Furthermore, by the above observation that $[\bLambda_K]_{sa} = N(s,a)$, achieving $\bphi(s,a)^\top \bLambda_K^{-1} \bphi(s,a) \le \bgamma_i^2$ is equivalent to setting $N = 1/\bgamma_i^2$. Since each $\cS_i$ can only be reached independently of the other $\cS_j$, $j \neq i$, this implies that to meet our objective we must run for at least
\begin{align*}
\textstyle\sum_{i = 1}^A 2^i \cdot \tfrac{ | \cS_i| A}{\bgamma_i^2} = \cO \Big ( \sum_{i=1}^A 2^i \cdot \tfrac{d}{\bgamma_i^2} \Big )
\end{align*}
episodes. This recovers the complexity \cdalg gets as given in \Cref{thm:policy_cover} (for $\bgamma_i^2$ sufficiently small). 
\end{exmp}

\subsection{Instantiating \cdalg with \algname}\label{sec:covtraj_force}
A recent work, \cite{wagenmaker2021first} provides a first-order regret minimization algorithm for linear MDPs, \algname. The following result shows the complexity of \cdalg when instantiated with \algname.

\begin{cor}
\label{cor: policy cover}
When running \cdalg with \regmin set to the computationally efficient version of \algname \citep{wagenmaker2021first}, all the guarantees of \Cref{thm:policy_cover} hold, and the sample complexity can be bounded as
\begin{align*}
\cOtil \bigg ( \sum_{i=1}^m 2^i & \cdot \max \bigg \{ \frac{d}{\bgamma^2_i}  \log \frac{2^i}{\bgamma^2_i},  d^4 H^3 m^3 \log^{7/2} \frac{1}{\delta}  \bigg \} \bigg )
\end{align*}
where the $\cOtil(\cdot)$ hides absolute constants and logarithmic terms. Furthermore, when run with \algname, \cdalg is computationally efficient with computational cost scaling polynomially in problem parameters.
\end{cor}

Thus, \cdalg may be instantiated in a computationally efficient way. We make several additional comments on the computational complexity of this approach. 

\paragraph{Computational Complexity.} 
Note that the primary computational cost of \cdalg is incurred in running \regmin, and in evaluating the reward function. As the sets $\cX_i$ can be parameterized as ellipsoids, it can be efficiently checked if a given $\bphi \in \R^d$ is in $\cX_i$, which makes evaluating the reward function efficient. Furthermore, \algname only needs access to the reward function at points encountered on each trajectory, so it only need evaluate the reward at polynomially many points. As noted, a computationally efficient version of \algname exists (assuming $|\cA|$ is not too large), which is what we employ here, making the entire procedure computationally efficient.

\subsection{Deriving \Cref{thm:pac} from \cdalg}
We briefly sketch out how one may apply \cdalg to obtain \Cref{thm:pac} and defer the full proof to \Cref{app:minimax}. Consider running \rfexp and then calling \pacalg. Let $V_h$ denote the estimates maintained by \pacalg. We first apply the following self-normalized bound.

\begin{lem}\label{lem:self_norm_informal}
With high probability,
\iftoggle{arxiv}{
\begin{align*}
& \bigg \| \sum_{i=1}^{\iotaeps} \sum_{\tau = 1}^{K_i} \bphi_{h,\tau}^i [ V_{h+1}(s_{h+1,\tau}^i) - \Exp_h[V_{h+1}](s_{h,\tau}^i,a_{h,\tau}^i)] \bigg \|_{\bLambda_h^{-1}}  \le c H \sqrt{d \log(1+dH\Kmax) + \log H/\delta} = \beta.
\end{align*}
}{
\begin{align*}
& \bigg \| \sum_{i=1}^{\iotaeps} \sum_{\tau = 1}^{K_i} \bphi_{h,\tau}^i [ V_{h+1}(s_{h+1,\tau}^i) - \Exp_h[V_{h+1}](s_{h,\tau}^i,a_{h,\tau}^i)] \bigg \|_{\bLambda_h^{-1}} \\
& \qquad \le c H \sqrt{d \log(1+dH\Kmax) + \log H/\delta} = \beta.
\end{align*}
}
\end{lem}
Critically, in contrast to the self-normalized bound used in \cite{jin2020provably}, ours scales as $\cOtil(H\sqrt{d})$ instead of $\cOtil(H d)$. As our exploration procedure explores each $h \in [H]$ separately, $V_{h+1}$ is \emph{uncorrelated} with $\{ \{ (s_{h,\tau}^i, a_{h,\tau}^i, s_{h+1,\tau}^i) \}_{\tau = 1}^{K_i} \}_{i =1}^{\iotaeps}$, and we can avoid the union bound over $\bLambda_{h+1}$ that is necessary in \cite{jin2020provably}, saving us a factor of $\sqrt{d}$. 

Given \Cref{lem:self_norm_informal}, using an argument similar to \cite{jin2020provably}, one can show that
\iftoggle{arxiv}{
\begin{align*}
V_h(s_h) & \le r_h(s_h,\pihat_h(s_h)) + \Exp_h[V_{h+1}](s_h,\pihat_h(s_h))  + 2\beta \| \bphi(s_h,\pihat_h(s_h)) \|_{\bLambda_h^{-1}} 
\end{align*}
}{
\begin{align*}
V_h(s_h) & \le r_h(s_h,\pihat_h(s_h)) + \Exp_h[V_{h+1}](s_h,\pihat_h(s_h)) \\
& \qquad + 2\beta \| \bphi(s_h,\pihat_h(s_h)) \|_{\bLambda_h^{-1}} 
\end{align*}
}
and, furthermore, that $V_0 \ge \Vst_0$. Some algebra shows that we can then bound the suboptimality of $\pihat$, the policy returned by \pacalg, as
\begin{align*}
\Vst_0 & - V_0^{\pihat}  \le V_0 - V_0^{\pihat}  \le 2 \beta \sum_{h=1}^H \Exp_{\pihat}[ \| \bphi(s_h,\pihat_h(s_h)) \|_{\bLambda_h^{-1}}] .
\end{align*}
However, it is easy to see that
\iftoggle{arxiv}{
\begin{align*}
  \sum_{h=1}^H \Exp_{\pihat}[ \| \bphi(s_h,\pihat_h(s_h)) \|_{\bLambda_h^{-1}}] & \le  \sum_{h=1}^H \sup_\pi \Exp_{\pi}[ \| \bphi(s_h,a_h) \|_{\bLambda_h^{-1}}] \\
& \le  \sum_{h=1}^H  \sum_{i=1}^{\iotaeps+1} \sup_{\bphi \in \cX_{hi}} \| \bphi \|_{\bLambda_h^{-1}} \cdot  \sup_\pi \Exp_{\pi}[ \I \{ \bphi(s_h,a_h) \in \cX_{hi} \}] .
\end{align*}
}{
\begin{align*}
&  \sum_{h=1}^H \Exp_{\pihat}[ \| \bphi(s_h,\pihat_h(s_h)) \|_{\bLambda_h^{-1}}] \\
 & \quad \le  \sum_{h=1}^H \sup_\pi \Exp_{\pi}[ \| \bphi(s_h,a_h) \|_{\bLambda_h^{-1}}] \\
& \quad \le  \sum_{h=1}^H  \sum_{i=1}^{\iotaeps+1} \sup_{\bphi \in \cX_{hi}} \| \bphi \|_{\bLambda_h^{-1}} \cdot  \sup_\pi \Exp_{\pi}[ \I \{ \bphi(s_h,a_h) \in \cX_{hi} \}] .
\end{align*}
}
Recall that \rfexp calls \cdalg for each $h \in [H]$ with tolerances $\bgamma_i^2 = \cO(\frac{2^{2i} \epsilon^2}{d H^4})$. Applying \Cref{thm:policy_cover}, we then have that $\sup_\pi \Exp_{\pi}[ \I \{ \bphi(s_h,a_h) \in \cX_{hi} \}] = \sup_\pi w_h^\pi(\cX_{hi}) \le 2^{-i + 1}$ and $\sup_{\bphi \in \cX_{hi}} \| \bphi \|_{\bLambda_h^{-1}} \le \sqrt{\bgamma_i^2} = \cO(2^i \epsilon / \sqrt{d} H^2)$. Thus, since $\beta = \cOtil(\sqrt{d} H)$, we can bound the total suboptimality as
\iftoggle{arxiv}{
\begin{align*}
\Vst_0  - V_0^{\pihat} &  \le 2 \beta \sum_{h=1}^H \Exp_{\pihat}[ \| \bphi(s_h,\pihat_h(s_h)) \|_{\bLambda_h^{-1}}] \\
& \le \cOtil(\sqrt{d} H) \cdot \sum_{h=1}^H \sum_{i=1}^{\iotaeps+1} \cO(2^i \epsilon / \sqrt{d} H^2) \cdot 2^{-i+1} \\
& = \cOtil(\epsilon)
\end{align*}
}{
\begin{align*}
\Vst_0 & - V_0^{\pihat}   \le 2 \beta \sum_{h=1}^H \Exp_{\pihat}[ \| \bphi(s_h,\pihat_h(s_h)) \|_{\bLambda_h^{-1}}] \\
& \le \cOtil(\sqrt{d} H) \cdot \sum_{h=1}^H \sum_{i=1}^{\iotaeps+1} \cO(2^i \epsilon / \sqrt{d} H^2) \cdot 2^{-i+1} \\
& = \cOtil(\epsilon)
\end{align*}
}
so it follows that $\pihat$ is $\epsilon$-optimal. 

\rev{We remark briefly on our choice of $\bgamma_i^2$. Note that as $i$ increases, the probability of reaching $\cX_{hi}$ decreases exponentially in $i$, $\sup_\pi \Exp_{\pi}[ \I \{ \bphi(s_h,a_h) \in \cX_{hi} \}] = \sup_\pi w_h^\pi(\cX_{hi}) \le 2^{-i + 1}$. Thus, if our ``uncertainty'' over $\cX_{hi}$ scales as $\cO(2^i)$, the net contribution of $\cX_{hi}$ to the suboptimality scales as $\cO(1)$. As we can bound our uncertainty over $\cX_{hi}$ as $\sup_{\bphi \in \cX_{hi}} \| \bphi \|_{\bLambda_h^{-1}} \le \sqrt{\bgamma_i^2}$, our choice of $\bgamma_i^2 = \cO(\frac{2^{2i} \epsilon^2}{d H^4})$ results in a contribution to the suboptimality of $\cO(\epsilon/\sqrt{d} H^2)$ for $\cX_{hi}$---our choice of $\bgamma_i^2$ cancels the contribution of how easily $\cX_{hi}$ can be reached. In other words, we only need to decrease uncertainty for a given $\cX_{hi}$ in proportion to how easily that $\cX_{hi}$ can be reached.

Furthermore, our choice of $\bgamma_i^2$ allows us to efficiently collect the samples necessary to reduce uncertainty. The leading-order term in the complexity given in \Cref{cor: policy cover} scales as, given our choice of $\bgamma_i^2$:
\begin{align*}
\cOtil \bigg ( \sum_{i=1}^m \frac{2^i d}{\bgamma_i^2} \bigg ) = \cOtil \bigg ( \sum_{i=1}^m \frac{d^2 H^4}{2^i \cdot \epsilon^2} \bigg ) = \cOtil \bigg ( \frac{d^2 H^4}{\epsilon^2} \bigg ),
\end{align*}
Repeated for each $h$ this yields the complexity given in \Cref{thm:pac}. 
The key property we exploit here is that \cdalg collects samples from a given $\cX_{hi}$ at a rate inversely proportional to how easily $\cX_{hi}$ can be reached---it takes on order $\cO(2^i)$ episodes for \cdalg to collect a sample from $\cX_{hi}$, while the probability of reaching $\cX_{hi}$ (for any algorithm) is at most on order $2^{-i}$. 
Thus, as our choice of $\bgamma_i^2$ only guarantees that we reduce uncertainty for $\cX_{hi}$ in proportion with how easily $\cX_{hi}$ can be reached, we have that the complexity of collecting the necessary samples is not prohibitively large, and in particular does not scale with the difficulty of reaching $\cX_{hi}$.}

\begin{rem}[Necessity of First-Order Regret]
Suppose that we instantiate \cdalg with a regret minimization algorithm which only achieves minimax regret, $\cR_K \le \cOtil(\sqrt{\cC_1 K})$, rather than first-order regret. In order to guarantee $\sup_{\pi} w_h^\pi(\cX_{i+1}) \le 2^{-i+2}$, we must show that $\cO(2^{-i} K_i) \ge \cR_{K_i}$, which ensures we have reached the ``difficult to reach'' states. Using a minimax regret algorithm we therefore need $K_i \ge \cOtil(2^{2i} \cC_1)$, while for a first-order algorithm, assuming $\sup_{\pi} w_h^\pi(\cX_{i}) \le 2^{-i+1}$, our choice of reward function gives $\Vst_0(r^1) \le \cO(2^{-i})$, so we need $\cO(2^{-i} K_i) \ge \cOtil(\sqrt{\cC_1 2^{-i} K_i}) \iff K_i \ge \cOtil(2^i \cC_1)$. 
This makes a critical difference when applying \cdalg to reward-free RL in \rfexp, since in \rfexp we set $m = \iotaeps = \cO(\log ( \sqrt{d}H^2/\epsilon))$. With this choice of $m$, a non-first-order algorithm would have complexity $\cOtil(\sum_{i=1}^m 2^{2i} \cdot \cC_1) = \cOtil(2^{2m} \cdot \cC_1)  = \cOtil(\cC_1 \frac{d^2 H^4}{\epsilon^2})$. As the best known minimax regret algorithm has $\cC_1 = d^2 H^4$, we obtain a (very suboptimal) final complexity of $\cOtil(\frac{d^4 H^8}{\epsilon^2})$.
On the other hand, a first-order algorithm attains $\cOtil(\sum_{i=1}^m 2^{i} \cdot \cC_1) = \cOtil(2^{m} \cdot \cC_1)  = \cOtil(\cC_1 \frac{d H^2}{\epsilon})$.

\end{rem}

\subsection{Well-Conditioned Covariates}

We conclude with an additional application of \cdalg to the problem of obtaining well-conditioned covariates. Several existing works \citep{hao2021online,agarwal2021online} assume access to a policy $\piexp$ able to collect covariates with minimum eigenvalue bounded away from 0, in order to ensure learning in every direction. However, to our knowledge, without access to such an oracle policy, there does not exist an algorithm able to provably collect such ``full-rank'' data. In the following result, we show that \cdalg can be used to collect such data, assuming it is possible.

\newcommand{\cDexp}{\cD_{\mathrm{exp}}}
\begin{thm}\label{thm:well_cond_cov}
Fix $h \in [H]$, $\gamma \in [0,1]$, and suppose $\sup_{\pi} \lambda_{\min}(\bbE_\pi [\bphi(s_h,a_h) \bphi(s_h,a_h)^\top]) \geq \epsilon $. 
Then there exists an algorithm which collects observations $\cDexp = \{ (s_{h,\tau},a_{h,\tau}) \}_{\tau=1}^K$ such that, with probability at least $1-\delta$:
\begin{align*}
     \lambda_{\min} \bigg ( \sum_{(s,a) \in \cDexp} \bphi(s,a)\bphi(s,a)^\top \bigg ) \geq \frac{\epsilon}{\gamma^2}
\end{align*}
after running for at most
\begin{align*}
K \le \cOtil \bigg ( \frac{1}{\epsilon} \cdot \max\bigg \{ \frac{d}{\gamma^2} , d^4 H^3  \log^{3} \frac{1}{\delta}  \bigg \} \bigg )
\end{align*}
episodes.
\end{thm}

\section{Conclusion}
In this work we have shown that in linear MDPs, reward-free RL is no harder than reward-aware RL. Along the way, we have developed a novel sample collection strategy that allows for efficient traversal of linear MDPs. Several questions remain open for future work. While this work establishes that $d^2$ is the optimal dimension-dependence for reward-aware (and reward-free) RL, our techniques do not directly provide a regret minimization algorithm. Developing a computationally efficient regret minimization algorithm with regret scaling as $\cO(\sqrt{d^2 \cdot \poly(H) \cdot K})$ would be an interesting direction to pursue. In addition, resolving the optimal $H$ dependence remains an open question. A second interesting direction would be to extend the work of \cite{tarbouriech2020provably} to the linear MDP setting. \cite{tarbouriech2020provably} provides an algorithm that allows the learner to collect an arbitrary number of samples for each state-action pair individually. In contrast, \cdalg only lets one specify the tolerance for each partition set as whole. The direct generalization of \cite{tarbouriech2020provably} to the linear setting would be to allow the learner to specify the tolerance in each direction individually. We believe \cdalg could be used as a basic building block in such an approach, but leave this extension for future work.

\subsection*{Acknowledgements}
The work of AW is supported by an NSF GFRP Fellowship DGE-1762114. The work of SSD is in part supported by grants NSF IIS-2110170. The work of KJ was funded in part by the AFRL and NSF TRIPODS 2023166.

\newpage
\bibliographystyle{icml2022}
\bibliography{bibliography.bib}

\newpage
\appendix


\section{Technical Results}\label{app:technical}
\begin{defn}[Covering Number]\label{defn:cov_num} Let $\calX$  be a set with metric $\dist(\cdot,\cdot)$. Given $\epsilon > 0$, the $\epsilon$-covering number of $\calX$ in $\dist$, $\covnum(\calX,\dist,\epsilon)$, is defined as the minimal cardinality of a set $\mathcal{N} \subset \calX$ such that, for all $x \in \calX$, there exists an $x' \in \calN$ with $\dist(x,x') \le \epsilon$. 
\end{defn}

\begin{lem}[\cite{vershynin2010introduction}]\label{lem:euc_ball_cover}
For any $\epsilon > 0$, the $\epsilon$-covering number of the Euclidean ball $\mathcal{B}^d(R) := \{\bx \in \R^d: \|\bx\|_2 =1\}$ with radius $R > 0$ in the Euclidean metric is upper bounded by $(1 + 2R/\epsilon)^d$. 
\end{lem}

\begin{lem}[Elliptic Potential Lemma, Lemma 11 of \cite{abbasi2011improved}]\label{lem:elip_pot}
Consider a sequence of vectors $( \bx_t)_{t=1}^T, \bx_t \in \bbR^d$, and assume that $\| \bx_t \|_2 \le a$ for all $t$. Let $\bV_t = \lambda I + \sum_{s=1}^t \bx_s \bx_s^\top$ for some $\lambda > 0$. Then we will have that
\begin{align*}
\sum_{t=1}^T \min \{ 1, \| \bx_t \|_{\bV_{t-1}^{-1}}^2 \} \le 2 d \log ( 1 + a^2 T / (d \lambda)) .
\end{align*}
Furthermore, if $\lambda \ge \max \{ 1, a^2 \}$,
\begin{align*}
\sum_{t=1}^T \| \bx_t \|_{\bV_{t-1}^{-1}}^2 \le 2 d \log ( 1 + a^2 T / (d \lambda)) .
\end{align*}
\end{lem}

\begin{lem}[Freedman's Inequality \citep{freedman1975tail}]\label{lem:freedman}
$\cF_0 \subset \cF_1 \subset \ldots \subset \cF_T$ be a filtration and let $X_1,X_2,\ldots,X_T$ be real random variables such that $X_t$ is $\cF_t$-measurable, $\Exp[X_t | \cF_{t-1}] = 0$, $| X_t | \le b$ almost surely, and $\sum_{t=1}^T \Exp[X_t^2 | \cF_{t-1}] \le V$ for some fixed $V > 0$ and $b > 0$. Then for any $\delta \in (0,1)$, we have with probability at least $1-\delta$,
\begin{align*}
\sum_{t=1}^T X_t \le 2\sqrt{V \log(1/\delta)} + b \log(1/\delta) .
\end{align*}
\end{lem}

\begin{lem}\label{claim:log_lin_burnin}
If $x \ge C (2n)^n \log^n(2n CB)$ for $n, C,B \ge 1$, then $x \ge C \log^n(B x)$.
\end{lem}
\begin{proof}
If $x = C (2n)^n \log^n(2n CB)$, then
\begin{align*}
C \log^n (Bx) & = C \log^n \big [ C (2n)^n \log^n(2n CB) \big ] \\
& \le C \log^n \big [ C^{1+n} (2n)^{2n} B^n  \big ] \\
& \le C \log^n \big [ C^{2n} (2n)^{2n} B^{2n}  \big ] \\
& \le C (2n)^n \log^n \big [ 2 n C B \big ] \\
& = x .
\end{align*} 
The result then follows since $x$ increases more quickly than $C \log^n (Bx)$. 
\end{proof}

\section{Collecting Covering Trajectories}\label{app:mdp_cover}

We prove a slightly more general version of \Cref{thm:policy_cover}. In particular, instead of setting $\bLambda_{h,k-1} = I + \sum_{\tau = 1}^{k-1} \bphi_{h,\tau} \bphi_{h,\tau}^\top$ as in \fes, we prove the result for the setting of $\bLambda_{h,k-1} = \lambda I + \sum_{\tau = 1}^{k-1} \bphi_{h,\tau} \bphi_{h,\tau}^\top$ in \fes, for some $\lambda > 0$ (for convenience we also assume $\lambda \le \max \{ \cC_1, \cC_2, 1 \}$, though this can be relaxed if desired). For the proof of \Cref{thm:pac} we simply set $\lambda = 1$.

In \cdalg, the precise setting of $K_i$ is
\begin{align*}
	K_i \leftarrow \bigg \lceil 2^i & \cdot \max \Big \{  2^{10+p_1} \cC_1 p_1^{p_1} \log^{p_1} ( \frac{2^{i+12} p_1 m \cC_1 H}{\delta}), 2^{4+p_2}  \cC_2 p_2^{p_2} \log^{p_2} ( \frac{2^{i+6} p_2 m \cC_2 H}{\delta} ),  \frac{24 d}{\bgamma_i^2} \log \frac{48 \cdot 2^i d/\lambda}{\bgamma_i^2} \Big \} \bigg \rceil .
\end{align*}

\begin{proof}[Proof of \Cref{thm:policy_cover}]
\Cref{thm:policy_cover} is a direct consequence of \Cref{lem:mdp_cover}. For $i = 1$, it is clearly the case that 
\begin{align*}
\sup_\pi w_h^\pi(\cX) \le 2^{-i + 1} = 1,
\end{align*}
so \Cref{lem:mdp_cover} and our choice of $K_1$ gives that with probability at least $1-\delta/m$, 
\begin{align*}
\sup_\pi w_h^\pi(\cX \backslash \cX_1) \le 2^{-i } \quad \text{and} \quad \bphi^\top \bLambda_1^{-1} \bphi \le \bgamma^2_1, \forall \bphi \in \cX_1 .
\end{align*}
Now assume that for some $i$, 
\begin{align*}
\sup_\pi w_h^\pi(\cX) \le 2^{-i + 1} ,
\end{align*}
then again \Cref{lem:mdp_cover} and our choice of $K_i$ gives that with probability at least $1-\delta/m$, 
\begin{align*} 
\sup_\pi w_h^\pi(\cX \backslash \cX_i) \le 2^{-i } \quad \text{and} \quad \bphi^\top \bLambda_i^{-1} \bphi \le \bgamma^2_i, \forall \bphi \in \cX_i .
\end{align*}
The result follows by union bounding over the success event of \Cref{lem:mdp_cover} holding for each $i \in [m]$. The final conclusion holds since after epoch $m$, $\cX = \cB^{d} \backslash \cup_{i=1}^m \cX_i$. 
\end{proof}

\begin{lem}\label{lem:mdp_cover}
Consider running \Cref{alg:find_exp_set} with $\gamma \in (0,1]$ and some regret-minimization algorithm \regmin satisfying \Cref{asm:regret_alg}, and with input set $\cX$ satisfying
\begin{align*}
\sup_\pi  \omega_h^\pi(\cX) \le 2^{-i} .
\end{align*}
Assume also that $K$ is chosen to satisfy 
\begin{align}\label{eq:cover_K_burnin}
\begin{split}
K \ge \bigg \lceil 2^i \cdot \max \Big \{ & 1024  \cC_1 \cdot (2p_1)^{p_1} \log^{p_1} \cdot [ 4096 \cdot 2^i p_1 \cC_1 H /\delta] ,  \\
& 16  \cC_2 \cdot (2p_2)^{p_2} \log^{p_2}  \cdot[ 64 \cdot 2^i p_2 \cC_2 H /\delta ] , \frac{24 d}{\gamma^2} \log \frac{48 \cdot 2^i d}{\gamma^2} \Big \} \bigg \rceil .
\end{split}
\end{align}
Let $\cXtil \subseteq \bbR^d$ denote the set returned by \Cref{alg:find_exp_set} defined as
\begin{align*}
\cXtil = \{ \bphi \in \cX \ : \ \bphi^\top \bLambda_{h,K}^{-1} \bphi \le \gamma^2 \} .
\end{align*}
Then, with probability at least $1-\delta$, 
\begin{align*}
\sup_\pi \omega_h^\pi(\cX \backslash \cXtil) \le 2^{-i-1} .
\end{align*}
\end{lem}
\begin{proof}
First note that the reward sequence used in \Cref{alg:find_exp_set} satisfies the conditions of \Cref{asm:regret_alg}. Thus, we have that, with probability at least $1-\delta$,
\begin{align}\label{eq:mdp_cover_regret1}
\sum_{k=1}^K [ \Vst_0(r^k) - V_0^{\pi_k}(r^k)] \le  \sqrt{\cC_1 \Vst_0(r^1) K \cdot \log^{p_1}(HK/\delta)} + \cC_2 \log^{p_2}(HK/\delta) .
\end{align}
For simplicity we will assume that $\cC_1, \cC_2, p_1, p_2 \ge 1$ (since if this is not true, for example if $\cC_1 < 1$, \eqref{eq:mdp_cover_regret1} still holds with $\cC_1$ replaced by $\max \{ \cC_1, 1 \}$).

\paragraph{Relating $\sum_{k=1}^K V_0^{\pi_k}(r^k)$ to random reward.}
Note that $r_h^k(s,a) \le 1$ for all $s,a$, and since the reward is non-zero only at step $h$, $\Exp_{\pi_k}[r_h^k(s_h^k,a_h^k)] = V_0^{\pi_k}(r^k)$ and $V_0^{\pi_k}(r_k) \le 1$. Then, since the reward is non-increasing,
\begin{align*}
\Exp[(r_h^k(s_h^k,a_h^k) - V_0^{\pi_k}(r^k))^2 | \cF_{k-1}] \le 2 V_0^{\pi_k}(r^k) \le 2\Vst_0(r^1).
\end{align*}
By Freedman's inequality, \Cref{lem:freedman}, it follows that with probability at least $1-\delta$,
\begin{align*}
\left | \sum_{k=1}^K [r_h^k(s_h^k,a_h^k) - V_0^{\pi_k}(r^k)] \right | \le \sqrt{8 \Vst_0(r^1) K \log 1/\delta} + \log 1/\delta .
\end{align*}
Thus, union bounding over this event and the event of \eqref{eq:mdp_cover_regret1}, we have  with probability $1-\delta$ that
\begin{align}
\sum_{k=1}^K r_h^k(s_h^k,a_h^k) & \ge \sum_{k=1}^K V_0^{\pi_k}(r^k) - \sqrt{8 \Vst_0(r^1) K \log 2/\delta} - \log 2/\delta \nonumber \\
& \ge \sum_{k=1}^K \Vst_0(r^k) - \sqrt{16 \cC_1 \Vst_0(r^1) K \cdot \log^{p_1}(2HK/\delta)} - 2 \cC_2 \cdot \log^{p_2}(2HK/\delta) \label{eq:reg_exp_in1}
\end{align}
where we have used that $\cC_1, \cC_2, p_1, p_2 \ge 1$ to group terms.

\paragraph{Proof by contradiction.} Our goal is use \eqref{eq:reg_exp_in2} to reach a contradiction and show that, for our choice of $K$, $\Vst_0(r^K) \le 2^{-i-1}$. To set up the argument, assume for the sake of contradiction that $\Vst_0(r^K) > 2^{-i-1}$. Since $r^k$ is non-increasing, this implies that $\Vst_0(r^k) > 2^{-i-1}$ for all $k \in [K]$. Then we can lower bound \eqref{eq:reg_exp_in1} as
\begin{align}\label{eq:reg_exp_in2}
\eqref{eq:reg_exp_in1} & > K 2^{-i-1}  - \sqrt{16 \cC_1 \Vst_0(r^1) K \cdot \log^{p_1}(2HK/\delta)} - 2 \cC_2 \cdot \log^{p_2}(2HK/\delta) .
\end{align}

\paragraph{Upper Bounding the Reward.} 
We next upper bound the total reward that can be obtained:
\begin{equation}\label{eq:rew:upper_bound}
\begin{aligned}
\sum_{k=1}^K r_h^k(s_h^k,a_h^k) & = \sum_{k=1}^K \min \{ 1, \gamma^{-2} \| \bphi(s_h^k,a_h^k) \|_{\bLambda_{h,k-1}^{-1}}^2 \} \cdot \I \{ \bphi(s_h^k,a_h^k) \in \cX \} \\
& \le \frac{1}{\gamma^2} \cdot \sum_{k=1}^K \min \{ 1,  \| \bphi(s_h^k,a_h^k) \|_{\bLambda_{h,k-1}^{-1}}^2 \} \cdot \I \{ \bphi(s_h^k,a_h^k) \in \cX \} \\
& \le \frac{1}{\gamma^2} \cdot \sum_{k=1}^K \min \{ 1,  \| \bphi(s_h^k,a_h^k) \|_{\bLambda_{h,k-1}^{-1}}^2 \} \\
& \le \frac{1}{\gamma^2}  \cdot 2 d \log(1 + K/d\lambda)
\end{aligned}
\end{equation}
where we have used that $\gamma \le 1$, and where the last inequality follows by the Elliptic Potential Lemma, \Cref{lem:elip_pot}, since $\| \bphi(s_h^k,a_h^k) \|_2 \le 1$ by assumption, and we normalize $\bLambda_{h,k}$ by $\lambda I$.

\paragraph{Lower Bounding the Reward.}
By assumption, we have that $\sup_\pi \omega_h^\pi(\cX) \le 2^{-i}$. This implies that
\begin{align*}
\Vst_0(r^1) & = \sup_\pi \Exp_{\pi}[ r_h^1(s_h,a_h)]  \le  \sup_\pi \Exp_\pi[ \I \{ \bphi(s_h,a_h) \in \cX \} ]  \le 2^{-i}
\end{align*}
where the last inequality follows since $\sup_\pi \Exp_\pi[ \I \{ \bphi(s_h,a_h) \in \cX \} ] = \sup_\pi \omega_h^\pi(\cX)$ by definition. Then,
\begin{align}\label{eq:reg_exp_in3}
\eqref{eq:reg_exp_in2} \ge  K 2^{-i-1} - \sqrt{2^{-i} \cdot 16 \cC_1 K \cdot \log^{p_1}(2HK/\delta)} - 2 \cC_2 \cdot \log^{p_2}(2HK/\delta) .
\end{align}
Assume that $K$ is chosen such that
\begin{align}\label{eq:cover_K_min1}
K \ge 1024 \cdot 2^i \cC_1 (2p_1)^{p_1} \log^{p_1} [ 4096 \cdot 2^i p_1 \cC_1 H /\delta]
\end{align}
then by \Cref{claim:log_lin_burnin} we will have
\begin{align*}
\frac{1}{4} K 2^{-i-1} - \sqrt{2^{-i} \cdot 16 \cC_1 K \cdot \log^{p_1}(2HK/\delta)} \ge 0.
\end{align*}
Similarly, if
\begin{align}\label{eq:cover_K_min2}
K \ge 16 \cdot 2^i \cC_2 (2p_2)^{p_2} \log^{p_2} [ 64 \cdot 2^i p_2 \cC_2 H /\delta ] 
\end{align}
then 
\begin{align*}
\frac{1}{4} K 2^{-i-1} - 2 \cC_2 \cdot \log^{p_2}(2HK/\delta) \ge 0.
\end{align*}
As \eqref{eq:cover_K_burnin} requires that $K$ is chosen so as to satisfy both \eqref{eq:cover_K_min1} and \eqref{eq:cover_K_min2}, it follows that
\begin{align*}
\eqref{eq:reg_exp_in3} \ge \frac{1}{2} K 2^{-i-1} .
\end{align*}
However, \eqref{eq:cover_K_burnin} also gives that $K \ge \frac{2^i \cdot 24 d}{ \gamma^2  } \log \frac{2^i \cdot 48 d/\lambda}{ \gamma^2  }$. \Cref{claim:log_lin_burnin} then implies that
\begin{align*}
K \ge \frac{2^i \cdot 12 d}{\gamma^2} \log (2K/\lambda)
\end{align*}
so that, stringing together the above inequalities,
\begin{align}
\sum_{k=1}^K r_h^k(s_h^k,a_h^k) \ge \eqref{eq:reg_exp_in1} > \eqref{eq:reg_exp_in2} \ge \eqref{eq:reg_exp_in3} \ge \frac{1}{2} K 2^{-i-1} \ge \frac{3 d}{ \gamma^2  } \log (2K/\lambda) . \label{eq:reg_exp_in_3a}
\end{align}

\paragraph{Concluding the proof.}
Combining \Cref{eq:reg_exp_in_3a,eq:rew:upper_bound}, we have shown that
\begin{align*}
\frac{2d}{\gamma^2} \log(1+K/d \lambda) \ge \sum_{k=1}^K r_h^k(s_h^k,a_h^k) > \eqref{eq:reg_exp_in3}  \ge \frac{3 d}{\gamma^2} \log (2K/\lambda).
\end{align*}
This is a contradiction, since $\frac{2d}{\gamma^2} \log(1+K/d\lambda) \le \frac{3 d}{\gamma^2} \log (2K/\lambda)$. Thus, with probability $1-\delta$, we must have that $\Vst_1(r^K) \le 2^{-i-1}$. The conclusion that $\sup_\pi \omega_h^\pi(\cX \backslash \cXtil) \le 2^{-i-1}$ follows on this event since
\begin{align*}
\Vst_0(r^K) & = \sup_\pi \Exp_\pi [ \I \{ \bphi(s_h,a_h)^\top \bLambda_{h,K-1}^{-1} \bphi(s_h,a_h) > \gamma^2, \bphi(s_h,a_h) \in \cX \} ] \\
& \qquad + \Exp_{\pi}[ \gamma^{-2} \bphi(s_h,a_h)^\top \bLambda_{h,K-1}^{-1} \bphi(s_h,a_h) \cdot \I \{\bphi(s_h,a_h)^\top \bLambda_{h,K-1}^{-1} \bphi(s_h,a_h) \le \gamma^2, \bphi(s_h,a_h) \in \cX\}  ]\\
& \ge \sup_\pi \Exp_\pi [ \I \{ \bphi(s_h,a_h)^\top \bLambda_{h,K}^{-1} \bphi(s_h,a_h) > \gamma^2, \bphi(s_h,a_h) \in \cX \} ] \\
& = \sup_\pi \Exp_\pi [ \I \{ \bphi(s_h,a_h) \in \cX \backslash \cXtil \}]
\end{align*}
where the final equality follows by definition of $\cXtil$.

\end{proof}

\subsection{\algname satisfies \Cref{asm:regret_alg}}
We invoke Theorem 8 of \cite{wagenmaker2021first}. To do so, we need to control an appropriate covering number.  Recall the ball $\cB^d := \{\bphi \in \R^d:\|\bphi\| \le 1\}$. Given a class of functions $\Fclass$ denote a set of functions $f:\ \cB^d \to \R$, we define the distance on $f_1,f_2 \in \Fclass$
\begin{align*}
\dist_{\infty}(f_1,f_2) := \sup_{\bphi \in \cB^d}|f_1(\bphi) - f_2(\bphi)|. 
\end{align*}

\begin{lem}\label{lem:Rfun_cover}
Consider the class of functions
\begin{align*}
\Rclass := \left \{ r : \Ball^d \rightarrow \R \ : \ r(\bphi) = \left \{ \begin{matrix} 1 & \| \bphi \|_{\bLambda^{-1}}^2 > \gamma^2, \bphi \in \cX \\
\gamma^{-2} \| \bphi \|_{\bLambda^{-1}}^2 & \| \bphi \|_{\bLambda^{-1}}^2 \le \gamma^2, \bphi \in \cX \\
0 & \bphi \not\in \cX \end{matrix} \right . , \bLambda \succeq I \right \} 
\end{align*}
for some $\gamma^2 > 0$ and $\cX \subseteq \R^d$. Then
\begin{align*}
\covnum(\Rclass,\dist_{\infty},\epsilon) \le d^2 \log \left ( 1 + \frac{2 \sqrt{d}}{\gamma^2 \epsilon} \right ).
\end{align*}
\end{lem}
\begin{proof}
Consider $r_1,r_2 \in \Rclass$ parameterized by $\bLambda_1,\bLambda_2$. Then,
\begin{align*}
\dist_{\infty}(r_1,r_2) & = \sup_{\bphi \in \cB^d} | r_1(\bphi) - r_2(\bphi) |  \le \sup_{\bphi \in \cB^d} \frac{1}{\gamma^2} \left | \| \bphi \|_{\bLambda_1^{-1}}^2 - \| \bphi \|_{\bLambda_2^{-1}}^2 \right | .
\end{align*}
The inequality follows easily by noting that in every possible case, we can bound
\begin{align*}
| r_1(\bphi) - r_2(\bphi) |  \le  \frac{1}{\gamma^2} \left | \| \bphi \|_{\bLambda_1^{-1}}^2 - \| \bphi \|_{\bLambda_2^{-1}}^2 \right | .
\end{align*}
Now note that
\begin{align*}
\left | \| \bphi \|_{\bLambda_1^{-1}}^2 - \| \bphi \|_{\bLambda_2^{-1}}^2 \right | = \left | \bphi^\top (\bLambda_1^{-1} - \bLambda_2^{-1} ) \bphi  \right | \le \| \bLambda_1^{-1} - \bLambda_2^{-1} \|_\op \le \| \bLambda_1^{-1} - \bLambda_2^{-1} \|_{\fro} . 
\end{align*}
Let $\cN$ be an $\gamma^2 \epsilon$ cover of $\{ \bA \in \R^{d \times d} \ : \ \| \bA \|_{\fro} \le \sqrt{d} \}$. Then by \Cref{lem:euc_ball_cover}, $\log | \cN | \le d^2 \log ( 1 + 2 \sqrt{d} / (\gamma^2 \epsilon))$. Furthermore, for any $\bLambda \succeq I$, we can find some $\bA \in \cN$ such that $\| \bLambda^{-1} - \bA \|_{\fro} \le \gamma^2 \epsilon$. Let
\begin{align*}
\mathscr{V} := \left \{ r(\cdot) \ : \ r(\bphi) = \left \{ \begin{matrix} 1 & \| \bphi \|_{\bA^{-1}}^2 > \gamma^2, \bphi \in \cX \\
\gamma^{-2} \| \bphi \|_{\bA}^2 & \| \bphi \|_{\bA}^2 \le \gamma^2, \bphi \in \cX \\
0 & \bphi \not\in \cX \end{matrix} \right . , \bA \in \cN  \right \} ,
\end{align*}
then it follows that $\mathscr{V}$ is an $\epsilon$-net of $\Rclass$ in the $\dist_{\infty}$ norm, and that $\log | \mathscr{V} | \le d^2 \log ( 1 + 2 \sqrt{d} / (\gamma^2 \epsilon))$. The result follows.
\end{proof}

\begin{lem}\label{lem:explore_regret}
Let $r^k$ be as defined in \fes. Then \algname satisfies \Cref{asm:regret_alg} with
\begin{align*}
& \cC_1 = d^4 H^3 \log(e + \sqrt{d} / \gamma^2), \quad p_1 = 3 \\
& \cC_2 = d^4 H^3 \log^{3/2}(e + \sqrt{d} / \gamma^2), \quad p_2 = 7/2
\end{align*}
\end{lem}
\begin{proof}
This follows directly from Theorem 8 of \cite{wagenmaker2021first} by noting that $r_h^k$ is non-increasing in $k$, $\cF_{k-1}$-measurable, $r_h^k \in \Rclass$, and using the covering number for $\Rclass$ given in \Cref{lem:Rfun_cover}. 
\end{proof}

\section{Minimax Optimal Reward-Free RL}\label{app:minimax}

Let us first establish notation. Recall the episode magnitudes $K_i$ from \Cref{alg:find_exp_set_levels}. We define
\newcommand{\Ktot}{K_{\mathrm{tot}}}
\begin{align*}
\Ktot :=  \sum_{i=1}^{\iotaeps} K_i, \quad \text{ where }\iotaeps := \ceil{\log_2(\frac{\epsilon}{4 \beta H})}
\end{align*}
We assume that our parameters $\Kmax$ and $\beta$ are chosen such that
\begin{align*}
\beta \ge \betatil := c H \sqrt{d \log(1 + dH\Ktot) + \log H/\delta}, \quad \Kmax \ge \Ktot.
\end{align*}
Throughout this section we will consider an arbitrary reward function satisfying \Cref{defn:linear_mdp}, and will let $\Vst,\Vpi,\Qst,\Qpi$ denote the value functions for $\pist$ and $\pi$, respectively, with respect to $r$. Similarly, we will let $V$ and $Q$ refer to the value function estimates maintained by \pacalg when run with $r$ as an input.

\begin{proof}[Proof of \Cref{thm:pac}]
First we establish $\epsilon$-suboptimality, then we address sample complexity.

\paragraph{$\epsilon$-suboptimality.} 
Fix some reward function $r$ satisfying \Cref{defn:linear_mdp}. Let $\pihat$ denote the policy returned by $\pacalg(r)$. As noted above, let $\Vst,\Vpi,V,\Qst,\Qpi,Q$ refer to the value functions with respect to $r$. 

Let $\cE_Q$ denote the high-probability even from \Cref{lem:self_norm}, which holds with probability at least $1-\delta$. 
By \Cref{lem:optimism} and the choice of parameter $\beta \ge \betatil $, the following holds on $\cE_Q$:
\begin{align*}
\Vst_0 - V_0^{\pihat} \le V_0 - V_0^{\pihat} .
\end{align*}

Moreover, by \Cref{lem:approx_exp} and, again using $\beta \ge \betatil$, we have
\begin{align*}\left | \innerb{\bphi(s,a)}{\what_h} - r_h(s,a) - \Exp_h[V_{h+1}](s,a) \right | \le \betatil \| \bphi(s,a) \|_{\bLambda_h^{-1}}  \le \beta \| \bphi(s,a) \|_{\bLambda_h^{-1}}.
\end{align*}
Thus, by definition of $Q$, we have
\begin{align*}
V_h(s_h) = Q_h(s_h,\pihat_h(s_h)) & \le \innerb{\bphi(s_h,\pihat_h(s_h))}{\what_h} + \beta \| \bphi(s_h,\pihat_h(s_h)) \|_{\bLambda_h^{-1}} \\
& \le r_h(s_h,\pihat_h(s_h)) + \Exp_h[V_{h+1}](s_h,\pihat_h(s_h)) + 2\beta \| \bphi(s_h,\pihat_h(s_h)) \|_{\bLambda_h^{-1}} .
\end{align*}
Similarly, by the Bellman Equation, we have
\begin{align*}
V_h^{\pihat}(s_h) = Q_h^{\pihat}(s_h,\pihat_h(s_h)) & = r_h(s_h,\pihat_h(s_h)) + \Exp_h[V_{h+1}^{\pihat}](s_h,\pihat_h(s_h)) .
\end{align*}
It follows that
\begin{align*}
V_h(s_h) - V_h^{\pihat}(s_h) \le \Exp_h[V_{h+1} - V_{h+1}^{\pihat}](s_h,\pihat_h(s_h)) + 2\beta \| \bphi(s_h,\pihat_h(s_h)) \|_{\bLambda_h^{-1}} .
\end{align*}
Thus, unrolling this backwards gives (since $\pihat$ is deterministic):
\begin{align*}
V_0 - V_0^{\pihat} & \le \Exp_1[V_2 - V_2^{\pihat}](s_1,\pihat_1(s_1)) + 2\beta \| \bphi(s_1,\pihat_1(s_1)) \|_{\bLambda_1^{-1}} \\
& \le \Exp_1[\Exp_2[V_3 - V_3^{\pihat}](s_2,\pihat_2(s_2)](s_1,\pihat_1(s_1)) + 2\beta \Exp_1[\| \bphi(s_2,\pihat_2(s_2)) \|_{\bLambda_2^{-1}}](s_1,\pihat_1(s_1)) \\
& \qquad \qquad + 2\beta \| \bphi(s_1,\pihat_1(s_1)) \|_{\bLambda_1^{-1}} \\
& = \Exp_{\pihat}[V_3(s_3) - V_3^{\pihat}(s_3)] + 2\beta \sum_{h=1}^2 \Exp_{\pihat}[\| \bphi(s_h,\pihat_h(s_h)) \|_{\bLambda_h^{-1}}] \\
& \vdots \\
& \le 2 \beta \sum_{h=1}^H \Exp_{\pihat}[ \| \bphi(s_h,\pihat_h(s_h)) \|_{\bLambda_h^{-1}}] .
\end{align*}
We then upper bound
\begin{align*}
2 \beta \sum_{h=1}^H \Exp_{\pihat}[ \| \bphi(s_h,\pihat_h(s_h)) \|_{\bLambda_h^{-1}}]  & \le 2 \beta \sum_{h=1}^H \sup_\pi \Exp_{\pi}[ \| \bphi(s_h,a_h) \|_{\bLambda_h^{-1}}] \\
&\le 2 \beta \sum_{h=1}^H  \sum_{i=1}^{\iotaeps+1} \sup_\pi \Exp_{\pi}[ \| \bphi(s_h,a_h) \|_{\bLambda_h^{-1}} \cdot \I \{ \bphi(s_h,a_h) \in \cX_{h,i} \}] \\
& \le 2 \beta \sum_{h=1}^H  \sum_{i=1}^{\iotaeps+1} \sup_{\bphi \in \cX_{h,i}} \| \bphi \|_{\bLambda_h^{-1}} \cdot  \sup_\pi \Exp_{\pi}[ \I \{ \bphi(s_h,a_h) \in \cX_{h,i} \}] .
\end{align*}
By \Cref{thm:policy_cover} and a union bound over $H$, we will have that, with probability at least $1-\delta$, for all $h \in [H]$ and $i \in [\iotaeps]$ simultaneously,
\begin{align*}
\sup_\pi \Exp_{\pi}[ \I \{ \bphi(s_h,a_h) \in \cX_{h,i} \}] & = \sup_\pi \omega_h^{\pi}(\cX_{h,i}) \le 2^{-i+1} . 
\end{align*}
Furthermore, \Cref{thm:policy_cover} also gives
\begin{align*}
 \sup_{\bphi \in \cX_{h,i}} \| \bphi \|_{\bLambda_h^{-1}}  \le  \sup_{\bphi \in \cX_{h,i}} \| \bphi \|_{\bLambda_{h,i}^{-1}} \le \sqrt{\bgamma^2_i } = \frac{2^i \epsilon}{8 H \iotaeps \beta}
\end{align*}
where the final equality follows by our setting of $\bgamma^2_i = \frac{2^{2i} \epsilon^2}{64 H^2 \iotaeps^2 \beta^2}$. Finally, one last invocation of \Cref{thm:policy_cover}, followed by the choice of $\iotaeps$, gives that
\begin{align*}
\sup_\pi \Exp_{\pi}[ \I \{ \bphi(s_h,a_h) \in \cX_{h,\iotaeps +1} \}] & = \sup_\pi \omega_h^{\pi}(\cX_{h,\iotaeps+1}) \le 2^{-\iotaeps} \le \frac{\epsilon}{4 \beta H}
\end{align*}
Lastly, observe that $\sup_{\bphi \in \cX_{h,\iotaeps +1}} \| \bphi \|_{\bLambda_h^{-1}} \le 1$ always holds, since $\bLambda_h \succeq I$ and $\| \bphi \|_2 \le 1$. This gives that
\begin{align*}
 2 \beta \sum_{h=1}^H  \sum_{i=1}^{\iotaeps+1} \sup_{\bphi \in \cX_{h,i}} \| \bphi \|_{\bLambda_h^{-1}} \cdot  \sup_\pi \Exp_{\pi}[ \I \{ \bphi(s_h,a_h) \in \cX_{h,i} \}] & \le 2  \beta \sum_{h=1}^H  \sum_{i=1}^{\iotaeps} 2^{-i+1} \frac{2^i \epsilon}{8 H \iotaeps \beta} + \frac{\epsilon}{2} \\
 & = \frac{\epsilon}{2} + \frac{\epsilon}{2}  = \epsilon
\end{align*}
so our policy $\pihat$ is $\epsilon$-optimal for the reward function $r$. However, since $r$ was arbitrary, the above holds for all $r$ satisfying \Cref{defn:linear_mdp}. 

\paragraph{Sample complexity.} It remains to bound the sample complexity. We note that we ran for a total of $\sum_{h=1}^{H} \sum_{i=1}^{\iotaeps} K_{i} = H\Ktot$ episodes, where again we recall
\begin{align*}
	K_{i} = \bigg \lceil 2^i \cdot \max \Big \{ & 1024 \cC_1 (2p_1)^{p_1} \log^{p_1} [ 4096 \cdot 2^i p_1 \iotaeps \cC_1 H^2 /\delta], \\
	&16  \cC_2 (2p_2)^{p_2} \log^{p_2} [ 64 \cdot 2^i p_2 \iotaeps \cC_2 H^2 /\delta ], \frac{24 d}{\bgamma^2_i} \log \frac{48 \cdot 2^i d}{\bgamma^2_i} \Big \} \bigg \rceil .
\end{align*}
As we use \algname as our regret minimization algorithm, using the values of $\cC_1,\cC_2,p_1,p_2$ given in \Cref{lem:explore_regret} we can bound (for a universal constant $c$)
\begin{align*}
K_{i} & \le c 2^i d^4 H^3 \log^{3/2}(\sqrt{d}/\bgamma^2_i) i^3 \log^{7/2} \left ( \iotaeps d H \log^{3/2}(\sqrt{d}/\bgamma^2_i) / \delta \right ) + 2^i \frac{24 d}{\bgamma^2_i} \log \frac{48 \cdot 2^i d}{\bgamma^2_i} \\
& \le 2^i d^4 H^3 \log^{7/2}(1/\delta) \cdot \poly \log(d,H,1/\epsilon, \log 1/\delta) + \frac{2^i d}{\bgamma^2_i} \cdot \poly \log(d,H,1/\epsilon, \log 1/\delta)
\end{align*}
where the last inequality follows by our choice of $\bgamma^2_i =   
 2^{2i}  \cdot \frac{\epsilon^2}{64 H^2 \iotaeps^2 \beta^2}$ and $ \iotaeps =    \lceil \log_2(4 \beta H/\epsilon) \rceil$, and upper bounding $\poly(i)$ factors by $\poly(\iotaeps)$. Let $\iota = \poly \log(d,H,1/\epsilon, \log 1/\delta)$ (the precise setting of which may change from line to line). Then we can bound the sample complexity by:
\begin{align*}
\sum_{h=1}^{H} \sum_{i=1}^{\iotaeps} K_{i} & \le \iota H \sum_{i=1}^{\iotaeps} \left ( 2^i d^4 H^3 \log^{7/2}(1/\delta) + \frac{2^i d}{\bgamma^2_i} \right ) \\
& \le \iota H \sum_{i=1}^{\iotaeps} \left ( 2^i d^4 H^3 \log^{7/2}(1/\delta) + \frac{2^i d H^2 \beta^2}{2^{2i} \epsilon^2} \right ) \\
& \le  \frac{d^4 H^5 \beta \log^{7/2}(1/\delta) \iota}{\epsilon} + \frac{d H^3 \beta^2 \iota}{\epsilon^2} \\
& \le \frac{d^{9/2} H^6 \log^{4}(1/\delta) \iota}{\epsilon} + \frac{d H^5 (d + \log 1/\delta) \iota}{\epsilon^2} . 
\end{align*}

\paragraph{Selecting $\beta$.} 
It remains to show that $\Kmax \ge \Ktot$, which will imply that our setting of $\beta$ in \pacalg satisfies $\beta \ge \betatil$. However, this follows directly by the above complexity bound, which is polynomial in all arguments, justifying our choice of $\Kmax$.

\end{proof}

\subsection{Concentration of Least Squares Estimates}

\begin{lem}[Lemma B.2 of \cite{jin2020provably}]\label{lem:what_norm_bound}
When running \pacalg, for all $h$,
\begin{align*}
\| \what_h \|_2 \le 2H \sqrt{d \Ktot} .
\end{align*}
\end{lem}
\begin{proof}
Note that our construction of $\what_h$ is identical to the construction of $\bw_h^K$ in \cite{jin2020provably}, so we can apply Lemma B.2 of \cite{jin2020provably}.
\end{proof}

\begin{lem}\label{lem:V_cov_num}
Consider the function class
\begin{align*}
\Fclass(\bLambda,\alpha) := \left \{ V(\cdot) = \min \left \{ \max_{a \in \cA} \innerb{\bw}{\bphi(\cdot,a)} + \betatil \| \bphi(\cdot,a) \|_{\bLambda^{-1}}, H \right \}, \| \bw \|_2 \le \alpha \right \}
\end{align*}
for some fixed $\Lambda \succ 0$ and $\betatil > 0$, $\alpha > 0$. Then
\begin{align*}
\covnum(\Fclass(\bLambda,\alpha),\dist_\infty,\epsilon) \le d \log(1 + 2 \alpha/\epsilon) .
\end{align*}
\end{lem}
\begin{proof}
Take some $V_1, V_2 \in \Fclass(\bLambda,\alpha)$. Since both $\min \{ \cdot, H \}$ and $\max_a$ are contraction maps, and $\bphi(s,a) \in \Ball^d$ for all $s,a$,
\begin{align*}
\sup_{s} | V_1(s) - V_2(s)| & \le \sup_{\bphi \in \Ball^{d}} | \innerb{\bw_1 - \bw_2}{\bphi} |  \le \| \bw_1 - \bw_2 \|_2 .
\end{align*}
Let $\cN$ denote an $\epsilon$-cover of the $\alpha$-ball, $\Ball^d(\alpha)$. By \Cref{lem:euc_ball_cover}, $\log | \cN | \le d \log(1 + 2 \alpha/\epsilon)$. By the above, we then have that for any $V \in \Fclass(\bLambda,\alpha)$, there exists some $V' \in \cN$ such that $\dist_\infty(V,V') \le \epsilon$, which completes the proof.
\end{proof}

\begin{lem}[Lemma D.4 of \cite{jin2020provably}]\label{lem:self_norm_cover}
Let $\{ s_\tau \}_{\tau = 1}^\infty$ be a stochastic process on state space $\cS$ with corresponding filtration $\{ \cF_\tau \}_{\tau=0}^{\infty}$. Let $\{ \bphi_\tau \}_{\tau = 0}^\infty$ be an $\R^d$-valued stochastic process where $\bphi_\tau \in \cF_{\tau-1}$ and $\| \bphi_\tau \|_2 \le 1$. Let $\bLambda_k = I + \sum_{\tau = 1}^k \bphi_\tau \bphi_\tau^\top$. Then for any $\delta > 0$, with probability at least $1-\delta$, for all $k \ge 0$, and any $V \in \Fclass$ so that $\sup_s | V(s) | \le H$, we have:
\begin{align*}
\left \| \bLambda_k^{-1/2} \sum_{\tau = 1}^k \bphi_\tau ( V(s_\tau) - \Exp[V(s_\tau) | \cF_{\tau - 1}] ) \right \|^2 \le 4 H^2 \left ( \frac{d}{2} \log (1 + k) + \log \frac{|\cN_\epsilon|}{\delta} \right ) + 8 k^2 \epsilon^2
\end{align*}
where $\cN_\epsilon$ is the $\epsilon$-covering of $\Fclass$ with respect to $\dist_\infty$. 
\end{lem}

\begin{lem}[Full version of \Cref{lem:self_norm_informal}]\label{lem:self_norm}
Let $\cEw$ denote the event that, for all $h \in [H]$ and $V \in \Fclass_{h+1}$ simultaneously,
\begin{align*}
\left \| \bLambda_h^{-1/2} \sum_{i=1}^{\iotaeps} \sum_{\tau = 1}^{K_i}  \bphi_{h,\tau}^i \left [ V(s_{h+1,\tau}^i) - \Exp_{h}[V](s_{h,\tau}^i,a_{h,\tau}^i) \right ] \right \|_2 \le c H \sqrt{d \log (1 + d H \Ktot) + \log H/\delta}
\end{align*}
for a universal constant $c$ and where $\Ktot = \sum_{i=1}^{\iotaeps} K_i$ and $\Fclass_{h+1} := \Fclass(\bLambda_{h+1},2H \sqrt{d \Ktot})$. Then $\bbP[\cEw] \ge 1-\delta$.
\end{lem}
\begin{proof}
This is a consequence of \Cref{lem:V_cov_num} and \Cref{lem:self_norm_cover}. Fix some $\epsilon$, then by \Cref{lem:V_cov_num} we have
\begin{align*}
\covnum(\Fclass_{h+1},\dist_\infty,\epsilon) \le d \log (1 + 2H \sqrt{d \Ktot}/\epsilon)
\end{align*}
and by \Cref{lem:self_norm_cover}, with probability $1-\delta$, for all $V \in \Fclass_{h+1}$ simultaneously,
\begin{align*}
\bigg \| \bLambda_h^{-1/2} \sum_{i=1}^{\iotaeps} \sum_{\tau = 1}^{K_i}  \bphi_{h,\tau}^i [ V(s_{h+1,\tau}^i) - & \Exp_{h}[V](s_{h,\tau}^i,a_{h,\tau}^i)  ] \bigg \|_2^2  \le 4 H^2 \left ( \frac{d}{2} \log (1 + \Ktot) + \log \frac{|\cN_\epsilon|}{\delta} \right ) + 8 \Ktot^2 \epsilon^2 \\
& \le 4  H^2 \left ( \frac{d}{2} \log (1 + \Ktot) + d \log (1 + \frac{2H \sqrt{d \Ktot}}{\epsilon}) + \log \frac{1}{\delta} \right ) + 8 \Ktot^2 \epsilon^2 .
\end{align*}
Note also that, due to our data collection procedure collecting samples independently for each $h$, we have that $\bLambda_h$ and $\{ \{ (\bphi_{h,\tau}^i, s_{h+1,\tau}^i) \}_{\tau = 1}^{K_i} \}_{i=1}^{\iotaeps}$ are uncorrelated with $\bLambda_{h+1}$, so, unlike in \cite{jin2020provably}, no union bound over possible $\bLambda_{h+1}$ is needed. Choosing $\epsilon = \sqrt{\frac{H^2 d}{8 \Ktot^2}}$, we can bound this as
\begin{align*}
\le c H^2 \left ( d \log(1 + d H \Ktot) + \log \frac{1}{\delta} \right )
\end{align*}
for a universal constant $c$. The result follows by a union bound over all $h$.
\end{proof}

\subsection{Optimism}
As noted above, throughout this section we will let $\Vst,\Vpi,\Qst,\Qpi$ denote value functions defined with respect to a generic reward function $r$, and $V,Q$ the value function estimates maintained by \pacalg when called with reward function $r$. 

\begin{lem}\label{lem:approx_exp}
On the event $\cEw$, for all $s,a,h$ and all $r$ satisfying \Cref{defn:linear_mdp},
\begin{align*}
\left | \innerb{\bphi(s,a)}{\what_h} - r_h(s,a) - \Exp_h[V_{h+1}](s,a) \right | \le \betatil \| \bphi(s,a) \|_{\bLambda_h^{-1}} 
\end{align*}
where $\betatil = c H \sqrt{d \log(1 + dH\Ktot) + \log H/\delta}$. 
\end{lem}
\begin{proof}
First, note that for any $r$ satisfying \Cref{defn:linear_mdp}, by \Cref{lem:what_norm_bound}, we will have that $V_h \in \Fclass_h$, where $\Fclass_h$ is defined as in \Cref{lem:self_norm}. Therefore, on $\cEw$, we have for any $r$,
\begin{align}\label{eq:cE_holds_bound}
\left \| \bLambda_h^{-1/2} \sum_{i=1}^{\iotaeps} \sum_{\tau = 1}^{K_i}  \bphi_{h,\tau}^i \left [ V_{h+1}(s_{h+1,\tau}^i) - \Exp_{h}[V_{h+1}](s_{h,\tau}^i,a_{h,\tau}^i) \right ] \right \|_2 \le \betatil.
\end{align}

By definition:
\begin{align*}
\what_h = \bLambda_h^{-1} \sum_{i=1}^{\iotaeps} \sum_{\tau = 1}^{K_i} \bphi_{h,\tau}^i (r(s_{h,\tau}^i,a_{h,\tau}^i) + V_{h+1}(s_{h+1,\tau}^i)) .
\end{align*}
Furthermore, we recall that $r_h(s_{h,\tau}^i,a_{h,\tau}^i) = \innerb{\btheta_h}{\bphi_{h,\tau}^i}$. Thus,
\begin{align*}
\innerb{\bphi(s,a)}{\what_h} & = \innerb{\bphi(s,a)}{\bLambda_h^{-1} \sum_{i=1}^{\iotaeps} \sum_{\tau = 1}^{K_i} \bphi_{h,\tau}^i (r(s_{h,\tau}^i,a_{h,\tau}^i) + V_{h+1}(s_{h+1,\tau}^i))} \\
& = \underbrace{\innerb{\bphi(s,a)}{\bLambda_h^{-1} \sum_{i=1}^{\iotaeps} \sum_{\tau = 1}^{K_i} \bphi_{h,\tau}^i (\bphi_{h,\tau}^{i})^\top \btheta_h}}_{(a)} \\
& \qquad + \underbrace{\innerb{\bphi(s,a)}{\bLambda_h^{-1} \sum_{i=1}^{\iotaeps} \sum_{\tau = 1}^{K_i} \bphi_{h,\tau}^i (V_{h+1}(s_{h+1,\tau}^i) - \Exp[V_{h+1}](s_{h,\tau}^i,a_{h,\tau}^i))}}_{(b)} \\
& \qquad + \underbrace{\innerb{\bphi(s,a)}{\bLambda_h^{-1} \sum_{i=1}^{\iotaeps} \sum_{\tau = 1}^{K_i} \bphi_{h,\tau}^i \Exp[V_{h+1}](s_{h,\tau}^i,a_{h,\tau}^i)}}_{(c)} .
\end{align*}
Now,
\begin{align*}
(a) = \innerb{\bphi(s,a)}{\btheta_h} -  \innerb{\bphi(s,a)}{\bLambda_h^{-1}  \btheta_h}  = r_h(s,a) -  \innerb{\bphi(s,a)}{\bLambda_h^{-1}  \btheta_h} 
\end{align*}
and, using the linear MDP assumption, \Cref{defn:linear_mdp},
\begin{align*}
(c) & = \innerb{\bphi(s,a)}{\bLambda_h^{-1} \sum_{i=1}^{\iotaeps} \sum_{\tau = 1}^{K_i} \bphi_{h,\tau}^i (\bphi_{h,\tau}^i)^\top \int V_{h+1}(s') \rmd \bmu_h(s')} \\
& = \innerb{\bphi(s,a)}{ \int V_{h+1}(s') \rmd \bmu_h(s')} -  \innerb{\bphi(s,a)}{\bLambda_h^{-1}  \int V_{h+1}(s') \rmd \bmu_h(s')} \\
& = \Exp_h[V_{h+1}](s,a) -  \innerb{\bphi(s,a)}{\bLambda_h^{-1}  \int V_{h+1}(s') \rmd \bmu_h(s')} .
\end{align*}
Thus,
\begin{align*}
\left | \innerb{\bphi(s,a)}{\what_h} - r_h(s,a) -  \Exp_h[V_{h+1}](s,a) \right | & \le (b) + \left | \innerb{\bphi(s,a)}{\bLambda_h^{-1}  \btheta_h}  \right | \\
& \qquad + \left |  \innerb{\bphi(s,a)}{\bLambda_h^{-1}  \int V_{h+1}(s') \rmd \bmu_h(s')} \right | .
\end{align*}
By Cauchy-Schwartz and \eqref{eq:cE_holds_bound}:
\begin{align*}
(b) & \le \| \bphi(s,a) \|_{\bLambda_h^{-1}} \left \| \bLambda_h^{-1/2} \sum_{i=1}^{\iotaeps} \sum_{\tau = 1}^{K_i} \bphi_{h,\tau}^i (V_{h+1}(s_{h+1,\tau}^i) - \Exp_h[V_{h+1}](s_{h,\tau}^i,a_{h,\tau}^i)) \right \|_2  \le \betatil \| \bphi(s,a) \|_{\bLambda_h^{-1}} .
\end{align*}
Furthermore, we can bound
\begin{align*}
\left |  \innerb{\bphi(s,a)}{\bLambda_h^{-1}  \btheta_h}  \right | \le \sqrt{ d}  \| \bphi(s,a) \|_{\bLambda_h^{-1}}
\end{align*}
and 
\begin{align*}
\left |  \innerb{\bphi(s,a)}{\bLambda_h^{-1}  \int V_{h+1}(s') \rmd \bmu_h(s')} \right |  & \le  \| \bphi(s,a) \|_{\bLambda_h^{-1}} \| \bLambda_h^{-1/2} \|_\op \|  \int V_{h+1}(s') \rmd \bmu_h(s') \|_2 \\
& \le H \sqrt{ d} \| \bphi(s,a) \|_{\bLambda_h^{-1}} 
\end{align*}
where the final inequality uses the linear MDP assumption, \Cref{defn:linear_mdp}, and that $\bLambda_h \succeq  I$. The result follows by combining these bounds. 
\end{proof}

\begin{lem}\label{lem:optimism}
On the event $\cEw$ and assuming that $\beta \ge \betatil$, it holds that $Q_h(s,a) \ge Q^{\star}_h(s,a)$ for all $s,a,h$.
\end{lem}
\begin{proof}
We will prove this by induction, starting at step $H$. Since the value function at $H+1$ is 0 by definition, $Q^{\star}_H(s,a) = r_H(s,a)$. Thus, \Cref{lem:approx_exp} gives
\begin{align*}
| \innerb{\bphi(s,a)}{\what_H} - Q^{\star}_H(s,a) | \le \betatil \| \bphi(s,a) \|_{\bLambda_H^{-1}} .
\end{align*}
It follows that, since $\beta \ge \betatil$,
\begin{align*}
Q^{\star}_H(s,a) \le \min \{  \innerb{\bphi(s,a)}{\what_H} + \beta \| \bphi(s,a) \|_{\bLambda_H^{-1}} , H \} = Q_H(s,a).
\end{align*}
Now assume that $Q_{h+1}(s,a) \ge Q^{\star}_{h+1}(s,a)$ for all $s,a$. By the Bellman equation,
\begin{align*}
Q^{\star}_h(s,a) = r_h(s,a) + \Exp_h[V^{\star}_{h+1}](s,a).
\end{align*}
Thus, \Cref{lem:approx_exp} gives
\begin{align*}
| \innerb{\bphi(s,a)}{\what_h} - Q^{\star}_h(s,a) - \Exp_h[V_{h+1} - V^{\star}_{h+1}](s,a) | \le \beta \| \bphi(s,a) \|_{\bLambda_h^{-1}} .
\end{align*}
By the inductive assumption, $\Exp_h[V_{h+1} - V^{\star}_{h+1}](s,a) \ge 0$, so we can rearrange this to get
\begin{align*}
Q^{\star}_h(s,a) \le \min \{ \innerb{\bphi(s,a)}{\what_h} + \betatil \| \bphi(s,a) \|_{\bLambda_h^{-1}}, H \} = Q_h(s,a).
\end{align*}
\end{proof}


\newcommand{\bz}{\bm{z}}

\section{Lower Bound for PAC Reinforcement Learning}

\begin{thm}[Lower Bound on Adaptive Linear Regression]\label{thm:regression_lb}
Let $\Phi = \cS^{d-1}$ and $\Theta = \{ - \mu, \mu \}^d$ for some $\mu \in (0,\frac{1}{20\sqrt{d}}]$. Consider the query model where at every step $t = 1,\ldots,T$ we choose some $\bphi_t \in \Phi$, and observe
\begin{align}\label{eq:reg_lb_query_model}
y_t \sim \bern(1/2 + \innerb{\btheta}{\bphi_t})
\end{align}
for $\btheta \in \Theta$. Then
\begin{align*}
\inf_{\bthetahat, \pi} \max_{\btheta \in \Theta} \Exp_{\btheta}[\| \btheta - \bthetahat \|_2^2] \ge \frac{d \mu^2}{2 } \left ( 1 - \sqrt{\frac{20 T \mu^2 }{d} }\right ).
\end{align*}
where the infimum is over all measurable estimators $\bthetahat$ and measurable (but possibly adaptive) query rules $\pi$, and $\Exp_{\btheta}[\cdot]$ denotes the expectation over the randomness in the observations and decision rules if $\btheta$ is the true instance.
In particular, if $T \ge d^2$, choosing $\mu = \sqrt{d/700T}$, we get
\begin{align*}
\inf_{\bthetahat, \pi} \max_{\btheta \in \Theta} \Exp_{\btheta}[\| \btheta - \bthetahat \|_2^2] \ge 0.00059 \cdot d^2/T  .
\end{align*}
\end{thm}
\begin{proof}
The proof of this result follows closely the proof of Theorem 3 of \cite{shamir2013complexity}. Clearly,
\begin{align*}
 \max_{\btheta \in \Theta} \Exp_{\btheta}[\| \btheta - \bthetahat \|_2^2] & \ge \Exp_{\btheta \sim \unif(\Theta)} \Exp_{\btheta}[\| \btheta - \bthetahat \|_2^2] \\
 & = \Exp_{\btheta \sim \unif(\Theta)} \Exp_{\btheta} \left [ \sum_{i=1}^d ( \btheta_i - \bthetahat_i)^2 \right ] \\
 & \ge \Exp_{\btheta \sim \unif(\Theta)} \Exp_{\btheta} \left [ \mu^2 \sum_{i=1}^d \I \{ \btheta_i \bthetahat_i < 0 \} \right ].
\end{align*}
As in \cite{shamir2013complexity}, we assume that the query strategy is deterministic conditioned on the past: $\bphi_t$ is a deterministic function of $y_1,\ldots,y_{t-1},\bphi_{1},\ldots,\bphi_{t-1}$, which is without loss of generality. We then need the following result.

\begin{lem}[Lemma 4 of \cite{shamir2013complexity}]\label{lem:reg_lb_shamir}
Let $\btheta$ be a random vector, none of whose coordinates is supported on 0, and let $y_1,y_2,\ldots,y_T$ be a sequence of query values obtained by a deterministic strategy returning a point $\bthetahat$ (that is, $\bphi_t$ is a deterministic function of $y_1,\ldots,y_{t-1},\bphi_{1},\ldots,\bphi_{t-1}$, and $\bthetahat$ is a deterministic function of $y_1,\ldots,y_T$). Then we have
\begin{align*}
\Exp_{\btheta \sim \unif(\Theta)}\Exp_{\btheta} \left [  \sum_{i=1}^d \I \{ \btheta_i \bthetahat_i < 0 \} \right ] \ge \frac{d}{2} \left ( 1 - \sqrt{\frac{1}{d} \sum_{i=1}^d \sum_{t=1}^T U_{t,i}} \right )
\end{align*}
where
\begin{align*}
U_{t,i} = \sup_{\btheta_j, j \neq i} \KL \Big (\Pr(y_t | \btheta_i > 0, \{ \btheta_j \}_{j \neq i}, \{ y_s \}_{s=1}^{t-1}) || \Pr(y_t | \btheta_i < 0, \{ \btheta_j \}_{j \neq i}, \{ y_s \}_{s=1}^{t-1}) \Big ) .
\end{align*}
\end{lem}

In our setting, $y_t$ is distributed as in \eqref{eq:reg_lb_query_model}. Thus, we have that
\begin{align*}
U_{t,i} = \sup_{\btheta_j, j \neq i} \KL \Big (\bern(1/2 + \sum_{j \neq i} \btheta_j \bphi_{t,j} + \mu \bphi_{t,i}) || \bern(1/2 + \sum_{j \neq i} \btheta_j \bphi_{t,j} - \mu \bphi_{t,i}) \Big ) .
\end{align*}
Note that:
\begin{lem}[Lemma 2.7 of \cite{tsybakov2009introduction}]\label{lem:bern_kl}
\begin{align*}
\KL(\bern(p) || \bern(q)) \le \frac{(p-q)^2}{q (1 - q)} .
\end{align*}
\end{lem}
Thus, by \Cref{lem:bern_kl}, we can bound
\begin{align*}
U_{t,i} & \le \frac{(2 \mu \bphi_{t,i})^2}{(1/2 + \sum_{j \neq i} \btheta_j \bphi_{t,j} - \mu \bphi_{t,i})(1 - 1/2 - \sum_{j \neq i} \btheta_j \bphi_{t,j} + \mu \bphi_{t,i})}  \le 20 \mu^2 \bphi_{t,i}^2
\end{align*}
where the last inequality follows by our assumption that $\mu \le \frac{1}{20\sqrt{d}}$ and since $\bphi_t \in \cS^{d-1}$, which implies that $1/2 + \sum_{j \neq i} \btheta_j \bphi_{t,j} - \mu \bphi_{t,i} \ge 9/20$ and $1 - 1/2 - \sum_{j \neq i} \btheta_j \bphi_{t,j} + \mu \bphi_{t,i} \ge 9/20$. 

By \Cref{lem:reg_lb_shamir} and this calculation, we can lower bound
\begin{align*}
\Exp_{\btheta \sim \unif(\Theta)} \Exp_{\btheta} \left [ \mu^2 \sum_{i=1}^d \I \{ \btheta_i \bthetahat_i < 0 \} \right ] & \ge \frac{d \mu^2}{2} \left ( 1 - \sqrt{\frac{1}{d} \sum_{i=1}^d \sum_{t=1}^T U_{t,i}} \right ) \\
& \ge  \frac{d \mu^2}{2 } \left ( 1 - \sqrt{\frac{1}{d} \sum_{i=1}^d \sum_{t=1}^T 20 \mu^2 \bphi_{t,i}^2} \right ) \\
& = \frac{d \mu^2}{2 } \left ( 1 - \sqrt{\frac{20 T \mu^2 }{d} }\right )
\end{align*}
where the final equality follows since $\bphi_t \in \cS^{d-1}$. This proves the first conclusion. The second conclusion follows by our choice of $\mu$.
\end{proof}

\begin{proof}[Proof of \Cref{thm:pac_lb}]
We consider the setting of instances in \Cref{thm:regression_lb} with $\Phi = \cS^{d-1}$ and $\Theta = \{ - \mu, \mu \}^d$. We first show how estimation error relates to finding $\epsilon$-good arms. Let $\bphist(\btheta)$ denote the optimal arm for $\btheta$. Clearly, $\bphist(\btheta) = \btheta/\| \btheta \|_2 = \btheta / (\sqrt{d} \mu)$ and $\bphist(\btheta)^\top \btheta = \sqrt{d}\mu$. Now consider a distribution $\pi \in \simplex_{\Phi}$. By definition, we have that $\pi$ is $\epsilon$-optimal if 
\begin{align*}
\Exp_{\bphi \sim \pi}[\btheta^\top \bphi] = \btheta^\top \Exp_{\bphi \sim \pi}[ \bphi]  \ge \sqrt{d} \mu - \epsilon .
\end{align*}
For a policy $\pi$, let $\bphi_\pi := \Exp_{\bphi \sim \pi}[ \bphi]$. Note that by the convexity of norms and Jensen's inequality: 
\begin{align*}
\| \bphi_\pi \|_2^2 \le \Exp_{\bphi \sim \pi}[ \| \bphi \|_2^2] = 1 .
\end{align*}
Write $\bphi_\pi = \bphist(\btheta) + \Delta$. Then,
\begin{align*}
1 \ge \| \bphist(\btheta) + \Delta \|_2^2 = 1 + \| \Delta \|_2^2 + 2 \bphist(\btheta)^\top \Delta \implies \bphist(\btheta)^\top \Delta \le -\frac{1}{2} \| \Delta \|_2^2 \implies \btheta^\top \Delta \le -\frac{\sqrt{d} \mu}{2} \| \Delta \|_2^2 . 
\end{align*}
However, $\bphi_\pi^\top \btheta = \bphist(\btheta)^\top \btheta + \Delta^\top \btheta$, so if $\pi$ is $\epsilon$-optimal, we have $\Delta^\top \btheta \ge -\epsilon$ which implies
\begin{align*}
 -\frac{\sqrt{d} \mu}{2} \| \Delta \|_2^2 \ge -\epsilon .
\end{align*}
In other words, if $\pi$ is $\epsilon$-optimal for $\btheta$, then $\bphi_\pi = \bphist(\btheta) + \Delta$ for some $\Delta$ with $\| \Delta \|_2^2 \le \frac{2 \epsilon}{\sqrt{d} \mu}$, so 
\begin{align*}
\bphi_\pi = \bphist(\btheta) + \Delta = \frac{\btheta}{\sqrt{d} \mu} + \Delta \implies \btheta = \sqrt{d} \mu (\bphi_\pi - \Delta) .
\end{align*}

Now assume that we have a $\pihat$ which is $\epsilon$-optimal, and denote $\bphihat := \bphi_{\pihat}$. Let $\btheta = \sqrt{d} \mu (\bphihat - \Delta)$ as above. Define the following estimator
\begin{align*}
\bthetahat = \left \{ \begin{matrix} \btheta' & \text{if } \exists \btheta' \in \Theta \text{ with } \btheta' = \sqrt{d} \mu(\bphihat - \Delta') \text{ for some } \Delta' \in \R^d, \| \Delta' \|_2^2 \le \frac{2\epsilon}{\sqrt{d} \mu} \\
\text{any } \btheta' \in \Theta & \text{otherwise} \end{matrix} \right .
\end{align*}
If $\bphihat$ is actually $\epsilon$-optimal for some $\btheta \in \Theta$, then the first condition is met and
\begin{align*}
\| \bthetahat - \btheta \|_2 = \| \sqrt{d} \mu ( \bphihat - \Delta') - \sqrt{d} \mu ( \bphihat - \Delta) \|_2 \le 2 \sqrt{d} \mu \sqrt{\frac{2 \epsilon}{\sqrt{d} \mu}} = \sqrt{8 \sqrt{d} \mu \epsilon}. 
\end{align*}
Thus, if we can find an $\epsilon$-good arm for $\btheta$, we can estimate $\btheta$ up to tolerance $\sqrt{8 \sqrt{d} \mu \epsilon}$.

Now set $\mu = \sqrt{d/700K}$. Let $\bthetahat$ denote the estimator constructed from $\bphihat$ as outlined above. Let $\cE$ be the event that $\bphihat$ is $\epsilon$-good for $\btheta$ and note that regardless of whether or not we are on $\cE$, $\| \bthetahat \|_2 \le d/\sqrt{700K}$. Thus,
\begin{align*}
\Exp_{\btheta}[\| \bthetahat - \btheta \|_2^2] & = \Exp_{\btheta}[\| \bthetahat - \btheta \|_2^2 \cdot \I \{ \cE \} + \| \bthetahat - \btheta \|_2^2 \cdot \I \{ \cE^c \} ] \\
& \le \frac{8 d \epsilon}{\sqrt{700K}} + \frac{4d^2}{700 K} \Pr_{\btheta}[\cE^c] 
\end{align*}
where the first inequality follows by our observation above that any $\epsilon$-good $\bphihat$ yields an estimate of $\btheta$ up to tolerance $\sqrt{2d\epsilon/\sqrt{K}}$.

However, by what we have shown above, there exists some $\btheta$ such that if we collect (no more than) $K$ samples,
\begin{align*}
\Exp_{\btheta} [  \| \bthetahat - \btheta \|_2^2] \ge 0.00059 \cdot d^2/K .
\end{align*}

This is a contradiction unless
\begin{align*}
\frac{8 d \epsilon}{\sqrt{700K}} + \frac{4d^2}{700 K} \Pr_{\btheta}[\cE^c] \ge 0.00059 \frac{d^2}{K} \iff \Pr_{\btheta}[\cE^c] \ge 0.10325 - \frac{2\sqrt{700} \epsilon \sqrt{K}}{d}. 
\end{align*}
It follows that if
\begin{align*}
0.10325 - \frac{2\sqrt{700} \epsilon \sqrt{K}}{d}\ge 0.1 \iff (\frac{0.00325}{2\sqrt{700}})^2 \cdot \frac{d^2}{\epsilon^2} \ge K
\end{align*}
we have that $\Pr_{\btheta}[\cE^c] \ge 0.1$.
\end{proof}

\subsection{Mapping to Linear MDPs}\label{sec:lin_bandit_lin_mdp}
We next show that the linear bandit instance of \Cref{thm:pac_lb} with parameter $\btheta$ (for $\btheta \in \Theta$ as in \Cref{thm:pac_lb}) can be mapped to a linear MDP with state space $\cS = \{ s_0, s_1, \sbar_2, \ldots, \sbar_{d+1} \}$, action space $\cA = \cS^{d-1} \cup \{ \be_{d+1}/2 \}$, parameters
\begin{align*}
\btheta_1 = \bm{0}, & \quad \btheta_h =  \be_1, h \ge 2 \\
\bmu_1(s_1) =  [2 \btheta, 1], & \quad \bmu_1(\sbar_i) = \frac{1}{d} [-2 \btheta, 1 ], \quad \bmu_h(s_i) = \be_i, h \ge 2
\end{align*}
 and feature vectors
\begin{align*}
& \bphi(s_0,\be_{d+1}/2) = \be_{d+1}/2, \quad \bphi(s_0,\bphitil) = [\bphitil/2,1/2], \quad \forall \bphitil \in \cS^{d-1} \\
& \bphi(s_1, \bphitil) = \be_1 ,  \quad \bphi(\sbar_i,\bphitil) =  \be_i, i \ge 2, \quad \forall \bphitil \in \cA.
\end{align*}
Note that, if we take action $\bphitil$ in state $s_0$, our expected episode reward is
\begin{align*}
P_1(s_1|s_0,\bphitil) \cdot H  + \sum_{i=2}^{d+1}  P_1(\sbar_i|s_0,\bphitil) \cdot 0 = H(\innerb{\btheta}{\bphitil} + 1/2)
\end{align*}
since we always acquire a reward of 1 in any state $s_1$, and a reward of 0 in any state $\sbar_i$, and the reward distribution is Bernoulli.

\begin{lem}
The MDP constructed above is a valid linear MDP as defined in \Cref{defn:linear_mdp}.
\end{lem}
\begin{proof}
For $\bphitil \in \cS^{d-1}$ we have,
\begin{align*}
P_1(s_1 | s_0, \bphitil) & = \innerb{\bphi(s_0,\bphitil)}{\bmu_1(s_1)} = \innerb{\btheta}{\bphitil} + 1/2 \ge 0 \\
P_1(\sbar_i | s_0, \bphitil) & = \innerb{\bphi(s_0,\bphitil)}{\bmu_1(\sbar_i)} = \frac{1}{d} ( -\innerb{\btheta}{\bphitil} + 1/2) \ge 0
\end{align*}
where the inequality follows since $|\innerb{\btheta}{\bphitil}| \le \cO(1/d)$ for all $\bphitil \in \cS^{d-1}$, by the choice of $\btheta$ in \Cref{thm:pac_lb}, and our condition that $K \ge d^2$.
In addition,
\begin{align*}
P_1(s_1|s_0,\bphitil)  + \sum_{i=2}^{d+1}  P_1(\sbar_i|s_0,\bphitil) =   \innerb{\bthetast}{\bphitil} + 1/2 + d \cdot \frac{1}{d} ( -\innerb{\bthetast}{\bphitil} + 1/2) = 1.
\end{align*}
Thus, $P_1(\cdot|s_0,\bphitil)$ is a valid probability distribution for $\bphitil \in \cS^{d-1}$. A similar calculation shows the same for action $\be_{d+1}/2$. It is obvious that $P_h(\cdot |s,\bphitil)$ is a valid distribution for all $s$ and $\bphitil \in \cA$.

It remains to check the normalization bounds. Clearly, by our construction of the feature vectors, $\| \bphi(s,a) \|_2 \le 1$ for all $s$ and $a$. It is also obvious that $\| \btheta_1 \|_2 \le \sqrt{d}$ and $\| \btheta_h \|_2 \le \sqrt{d}$. Finally,
\begin{align*}
\| | \bmu_1(\cS) | \|_2 = \left \| \sum_{s \in \cS \backslash s_0} |\bmu_1(s)| \right \|_2 = \| [2 \btheta, 1] + d \cdot \frac{1}{d} [2 \btheta, 1] \|_2 \le \sqrt{d}.
\end{align*}
Thus, all normalization bounds are met, so this is a valid linear MDP. 
\end{proof}

\newcommand{\bphibar}{\bar{\bphi}}
\paragraph{Lower bounding the performance of low-regret algorithms.}

Assume that we have access to the linear bandit instance constructed in  \Cref{thm:pac_lb}. That is, at every timestep $t$ we can choose an arm $\bphitil_t \in \cS^{d-1}$ and obtain and observe reward $y_t \sim \bern(\innerb{\btheta}{\bphitil_t} + 1/2)$. Using the mapping above, we can use this bandit to simulate a linear MDP as follows:
\begin{enumerate}
\item Start in state $s_0$ and choose any action $\bphitil_t \in \cA$ 
\item Play action $\bphitil_t$ in our linear bandit. If reward obtained is $y_t = 1$, then in the MDP transition to any of the states $s_1$. If the reward obtained is $y_t = 0$ transition to any of the states $\sbar_2,\ldots,\sbar_{d+1}$, each with probability $1/d$. If the chosen action was $\bphitil_t = \be_{d+1}/2$, then play any action in the linear bandit and transition to state $s_1$ with probability 1/2 and $\sbar_2,\ldots,\sbar_{d+1}$ with probability $1/2d$, regardless of $y_t$
\item For the next $H-1$ steps, take any action in the state in which you end up, transition back to the same state with probability 1, and receive reward of 1 if you are in $s_1$, and reward of 0 if you are in $\sbar_2,\ldots,\sbar_{d+1}$.
\end{enumerate}
Note that this MDP has precisely the transition and reward structure as the MDP constructed above.

\begin{lem}\label{lem:lin_ban_to_lin_mdp_policy}
Assume $\pi$ is $\epsilon$-optimal in the MDP constructed above. Then $\pi_1(\cdot |s_0)$ is $\epsilon/H$-optimal on the linear bandit instance with parameter $\btheta$ and action set $\cS^{d-1}$. 
\end{lem}
\begin{proof}
Note that the value of $\pi$ in the linear MDP is given by 
\begin{align*}
\Vpi_0 = H \cdot \left ( \sum_{\bphitil \in \cS^{d-1}} \pi_1(\bphitil|s_0) (\innerb{\bphitil}{\btheta} + 1/2) + \pi_1(\be_{d+1}/2|s_0)/2 \right ) = H \cdot \left (  \sum_{\bphitil \in \cS^{d-1}} \pi_1(\bphitil|s_0) \innerb{\bphitil}{\btheta} + 1/2 \right )
\end{align*}
and the optimal policy is $\pi_1(\ast |s_0) = 1$, for $\ast = \argmax_{\bphitil \in \cS^{d-1}} \innerb{\bphitil}{\btheta}$, and has value $\Vst_0 = \inner{\ast}{\bthetast} + 1/2$. It follows that if $\pi$ is $\epsilon$-optimal, then 
\begin{align*}
& H \cdot \left (  \sum_{\bphitil \in \cS^{d-1}} \pi_1(\bphitil|s_0) \innerb{\bphitil}{\btheta} + 1/2 \right ) \ge H \cdot \left ( \innerb{\ast}{\btheta} + 1/2 \right ) - \epsilon \\
& \iff   \sum_{\bphitil \in \cS^{d-1}} \pi_1(\bphitil|s_0) \innerb{\bphitil}{\btheta} + 1/2  \ge  \innerb{\ast}{\btheta} + 1/2  - \epsilon/H
\end{align*}
This implies that $\pi_1(\cdot | s_0)$ is $\epsilon/H$-optimal for the linear bandit with parameter $\btheta$ and action set $\cS^{d-1}$ (note that any mass $\pi(\cdot | s_0)$ places on arm $\be_{d+1}/2$ can be instead allocated to any arm in $\cS^{d-1}$ without changing this result).
\end{proof}

\begin{proof}[Proof of \Cref{cor:lin_mdp_lb}]
Consider running the above procedure for some number of steps.
By \Cref{lem:lin_ban_to_lin_mdp_policy}, if we can identify an $\epsilon$-optimal policy in this MDP, we can use it to determine an $\epsilon/H$-optimal policy on our bandit instance. As we have used no extra information other than samples from the linear bandit to construct this, it follows that to find an $\epsilon/H$-optimal policy in the MDP, we must take at least the number of samples prescribed by \Cref{thm:pac_lb} for $\epsilon \leftarrow \epsilon/H$. 
\end{proof}


\section{Well-Conditioned Covariates}

\begin{algorithm}[h]
\caption{Collect Well-Conditioned Covariates}
\begin{algorithmic}[1]
\State \textbf{input:} confidence $\delta$, target minimum reachability $\epsilon$, step $h$, tolerance $\gamma^2$
\State Set $m \leftarrow \log(2/\epsilon)$ and $\lambda \leftarrow \min \{ 1, \frac{\epsilon}{4 m \gamma^2 } \}$
\Statex \algcomment{Use $\bLambda_{h,k-1} = \lambda I + \sum_{\tau = 1}^{k-1} \bphi_{h,\tau} \bphi_{h,\tau}^\top$ in \fes}
	\State $\{ (\cX_{i},\cD_{i},\bLambda_{i}) \}_{i=1}^m \leftarrow \cdalg(h,\delta,\gamma^2, m, \algname)$
\State \textbf{return} $\cD_\text{exp} = \bigcup_{i=1}^m \cD_i $
\end{algorithmic}
\label{alg:well_cond_cov}
\end{algorithm}

\begin{rem}[Handling unknown $\epsilon$]
Note that \Cref{alg:well_cond_cov} requires knowledge of $\epsilon$, a lower bound on the achievable minimum eigenvalue. In general a tight lower bound may not be known. In such situations, one can repeatedly run \Cref{alg:well_cond_cov} starting with $\epsilon = 1/d$ and halving $\epsilon$ each time until covariates satisfying the desired lower bound are obtained. Once $\epsilon$ reaches an achievable level, \Cref{thm:well_cond_cov} will apply and the desired covariates will be collected. Note that this procedure only adds a complexity of $\cO(\log(1/\epsilon_0))$ to the complexity, where $\epsilon_0 = \sup_\pi \lambda_{\min} (\Exp_\pi [ \bphi(s_h,a_h) \bphi(s_h,a_h)^\top])$. 
\end{rem}

\begin{proof}[Proof of \Cref{thm:well_cond_cov}]

\textbf{Step 1: Collected Data is Full-Rank.}
We first argue that $\{\calX_i\}_{i=1}^m$ is well-spread on full $d$ dimension. For the following, let $\cY = \cup_{i=1}^m \cX_i$. 
We have, for any $\bv \in \cS^{d-1}$,
\begin{align*}
    \sup_{\pi} \bv^\top & \left( \E_{\pi}[ \bphi(s_h,a_h) \bphi(s_h,a_h)^\top ]  \right) \bv \\
    &  \le \sup_{\pi} \bv^\top  \left( \E_{\pi}[ \bphi(s_h,a_h) \bphi(s_h,a_h)^\top \cdot \I \{ \bphi(s_h,a_h) \in \cY\} ]  \right) \bv \\
    & \qquad + \sup_{\pi} \bv^\top  \left( \E_{\pi}[ \bphi(s_h,a_h) \bphi(s_h,a_h)^\top \cdot \I \{ \bphi(s_h,a_h) \in \cX_{m+1} \} ]  \right) \bv \\
    & \leq \sup_{\pi} \bv^\top  \left( \E_{\pi}[ \bphi(s_h,a_h) \bphi(s_h,a_h)^\top \cdot \I \{ \bphi(s_h,a_h) \in \cY\} ]  \right) \bv + \epsilon/2, 
\end{align*}
where the last inequality comes from the definition of $\calX_{m+1}$, \Cref{thm:policy_cover}, and our choice of $m$. On the other hand, by assumption we have 
\begin{align*}
\sup_{\pi} \bv^\top & \left( \E_{\pi}[ \bphi(s_h,a_h) \bphi(s_h,a_h)^\top ]  \right) \bv \ge \sup_{\pi} \lambda_{\min}  \left( \E_{\pi}[ \bphi(s_h,a_h) \bphi(s_h,a_h)^\top ]  \right) \ge \epsilon
\end{align*}
Therefore, we must have
\begin{align*}
    \max_{\bphi\in \cY } \|\bphi \|_{\bv \bv^\top}^2 
    \geq \sup_{\pi} \bv^\top  \left( \E_{\pi}[ \bphi(s_h,a_h) \bphi(s_h,a_h)^\top \cdot \I \{ \bphi(s_h,a_h) \in \cY\} ]  \right) \bv
    \geq \epsilon/2
\end{align*}

\textbf{Step 2: Lower Bounding Eigenvalue.}
The next step is to show the lower bound of the eigenvalue of $\sum_{i=1}^m \bLambda_{i}$, which is equivalent to upper bounding the eigenvalue of  $\left(\sum_{i=1}^m \bLambda_{i}\right)^{-1}$.
 
Denote the eigenvectors of $\sum_{i=1}^m \bLambda_{i}$ as $\{\bu_j\}_{j=1}^d$. Now for any $\bu_i$ and any corresponding $\bphi = \argmax_{\bphi'\in \cY} \|\bphi'\|_{\bu_j \bu_j^\top}^2$, we have
\begin{align*}
    \bu_j^\top \left(\sum_{i=1}^m \bLambda_i \right)^{-1} \bu_j
    & =  \frac{1}{\|\bphi\|^2_{\bu_j \bu_j^\top}}(\bu_j^\top \bphi) \bu_j^\top \left(\sum_{i=1}^m \bLambda_i\right)^{-1} (\bu_j^\top \bphi) \bu_i  \\
    &\leq  \frac{1}{\|\bphi\|^2_{\bu_j \bu_j^\top}} \bphi^\top \left(\sum_{i=1}^m \bLambda_i\right)^{-1} \bphi \\
    & \leq 2\epsilon^{-1} \bphi^\top \left(\sum_{i=1}^m \bLambda_i\right)^{-1} \bphi  \\
    & \leq 2\epsilon^{-1} \bphi^\top \left(\sum_{i=1}^m \bLambda_i \one\{\bphi \in \calX_i \}\right)^{-1} \bphi \\
   &  \leq  2 \gamma^2 \epsilon^{-1}
\end{align*}
where the last inequality comes from Theorem~\ref{thm:policy_cover} and the second inequality holds since, writing $\bphi = a \bu_j + b \bv$ for $\bv^\top \bu_j = 0$, we have
\begin{align*}
\bphi^\top \left(\sum_{i=1}^m \bLambda_i\right)^{-1} \bphi & = a^2 \bu_j^\top \left(\sum_{i=1}^m \bLambda_i\right)^{-1} \bu_j + \underbrace{b^2 \bv^\top \left(\sum_{i=1}^m \bLambda_i\right)^{-1} \bv}_{\ge 0} + \underbrace{2 a b \bu_j^\top \left(\sum_{i=1}^m \bLambda_i\right)^{-1} \bv}_{= 0 }.
\end{align*}

\textbf{Step 3: Concluding the Proof.}
Now we are ready for final proof. Using what we have just shown, we have for any $\bu_j$:

\begin{align*}
    \bu_j^\top\left(\sum_{i=1}^m \sum_{(s_{h,\tau},a_{h,\tau}) \in \calD_i} \bphi(s_{h,\tau},a_{h,\tau}) \bphi(s_{h,\tau},a_{h,\tau})^\top\right) \bu_j
    = \bu_j^\top\left(\sum_{i=1}^m \bLambda_i \right) \bu_j - m \lambda
    \geq \left(\frac{\epsilon}{2\gamma^2} - \lambda m \right)    
    \geq \frac{\epsilon}{4 \gamma^2}.
\end{align*}
The last inequality holds since $\lambda = \min \{ 1, \frac{\epsilon}{4 m \gamma^2 } \}$.
Now notice that from \Cref{cor: policy cover}, we have total sample complexity at most
\begin{align*}
\cOtil \bigg ( \sum_{i=1}^m 2^i & \cdot \max \bigg \{ \frac{d}{\gamma^2_i}  \log \frac{2^i/\lambda}{\gamma^2_i},   d^4 H^3 m^3 \log^{3} \frac{1}{\delta}  \bigg \} \bigg ) .
\end{align*}
Using $m = \log(2/\epsilon)$ and $\lambda = \min \{ 1, \frac{\epsilon}{4 m \gamma^2 } \}$, we have the complexity upper bound for any $h$ of
\begin{align*}
\cOtil \bigg ( \frac{1}{\epsilon} \cdot \max \bigg \{ \frac{d}{\gamma^2}, d^4 H^3  \log^{3} \frac{1}{\delta}  \bigg \} \bigg ) .
\end{align*}

\end{proof}

\end{document}